\documentclass[conference]{IEEEtran}
\IEEEoverridecommandlockouts
\title{Gated Information Bottleneck for Generalization in Sequential Environments}

\author{Francesco Alesiani, Shujian Yu, Xi Yu}
\author{
    \IEEEauthorblockN{Francesco Alesiani\IEEEauthorrefmark{1}, Shujian Yu\IEEEauthorrefmark{2}\IEEEauthorrefmark{3}, Xi Yu\IEEEauthorrefmark{4}}
    \IEEEauthorblockA{\IEEEauthorrefmark{1} NEC Laboratories Europe
    }
    \IEEEauthorblockA{\IEEEauthorrefmark{2} UiT - The Arctic University of Norway
    }
    \IEEEauthorblockA{\IEEEauthorrefmark{3} Xi'an Jiaotong University
    }
    \IEEEauthorblockA{\IEEEauthorrefmark{4} University of Florida
    }
    \thanks{Francesco Alesiani (Francesco.Alesiani@neclab.eu) and Shujian Yu (yusj9011@gmail.com) contributed equally and are contact authors.}
}

\usepackage[numbers]{natbib}
\usepackage{graphicx}

\usepackage{url}
\usepackage{amsmath}
\usepackage{amsfonts}
\usepackage{amssymb}
\usepackage{amsthm}
\usepackage{xcolor}
\usepackage{mathrsfs}
\usepackage{mathtools}
\usepackage{subfigure}
\usepackage{gensymb}
\usepackage{amsmath}
\usepackage{amsmath,amsfonts,amssymb}
\usepackage{graphicx}
\usepackage{bm}
\usepackage{algorithm,algpseudocode,float}

\usepackage{threeparttable}
\usepackage{multirow}
\usepackage{booktabs}
\usepackage{tablefootnote}
\usepackage{array}
\usepackage{caption}
\usepackage{mathtools}

\DeclarePairedDelimiter{\diagfences}{(}{)}

\newcommand{\tr}{\operatorname{tr}\diagfences}

\newcommand{\KL}{\operatorname{KL}\diagfences}

\DeclareMathOperator{\E}{\mathbb{E}}

\usepackage{bbm}

\usepackage{mathtools}

\usepackage{centernot}

\usepackage{wrapfig}
\usepackage{tikz}
\usepackage[american]{babel}
\usepackage{mathtools} % amsmath with fixes and additions
\usepackage{booktabs} % commands to create good-looking tables
\usepackage{tikz} % nice language for creating drawings and diagrams

\newcommand\independent{\protect\mathpalette{\protect\independenT}{\perp}}
\def\independenT#1#2{\mathrel{\rlap{$#1#2$}\mkern2mu{#1#2}}}

\allowdisplaybreaks

\usepackage{thmtools}
\usepackage{thm-restate}
\usepackage{hyperref}
\usepackage{cleveref}

\declaretheorem[name=Lemma]{lem}
\makeatletter
\usepackage{comment}
\let\wfs@comment@comment\comment
\let\comment\@undefined
\usepackage{changes}
\let\wfs@changes@comment\comment
\let\comment\@undefined
\newcommand\comment{%
    \ifthenelse{\equal{\@currenvir}{comment}}
    {\wfs@comment@comment}
    {\wfs@changes@comment}%
}
\makeatother

\begin{document}

\maketitle

\begin{abstract}
%Modern machine learning algorithms suffer from poor generalization ability to unseen environments when the underlying data distribution is different from that in the training set. One of the guiding principles for generalization is to learn invariant representations across different environments, as suggested by \cite{peters2017elements}.
% , such that ....
%In this work, we propose a new information-theoretic approach, termed the Information Guillotine (IG), for invariant representation learning.  IG closely resembles the renowned Information Bottleneck (IB) approach, but adjusts the dimension of the feature space dynamically.  We apply IG to discover environment-invariant representations. Particularly, we observed that IG demonstrated appealing performance when different environments are observed sequentially, a more practical scenario where invariant risk minimization (IRM), a recent approach proposed by
% Arjovsky et al. (2019)~\cite{arjovsky2019invariant}, fails.

Deep neural networks suffer from poor generalization to unseen environments when the underlying data distribution is different from that in the training set. By learning minimum sufficient representations from training data, the information bottleneck (IB) approach has demonstrated its effectiveness to improve generalization in different AI applications. In this work, we propose a new neural network-based IB approach, termed gated information bottleneck (GIB), that dynamically drops spurious correlations and progressively selects the most task-relevant features across different environments by a trainable \textit{soft mask} (on raw features). GIB enjoys a simple and tractable objective, without any variational approximation or distributional assumption. We empirically demonstrate the superiority of GIB over other popular neural network-based IB approaches in adversarial robustness and out-of-distribution (OOD) detection. Meanwhile, we also establish the connection between IB theory and invariant causal representation learning, and observed that GIB demonstrates appealing performance when different environments arrive sequentially, a more practical scenario where invariant risk minimization (IRM) fails. Code of GIB is available at \url{https://github.com/falesiani/GIB}.

% ~\cite{arjovsky2019invariant}, a recent proposal by Arjovsky \emph{et al.},

%\commentSY{Gated IB or Masked IB?}

% Arjovsky et al. (2019)~\cite{arjovsky2019invariant}, fails.

\end{abstract}
% We introduce the Information Guillotine principle and apply it to invariant feature learning. We shown as this approach is particular appealing in sequential task learning, but we also shown its application to parallel task learning. We draw relationship with the Varitioanl Information Bottelneck method and compare performances.

% \commentSY{Modern machine learning algorithms suffer from poor generalization ability to unseen environments when the underlying data distribution is different from that in the training set. One of the guiding principles for generalization is to learn invariant representations across different environments, such that .... In this work, we propose a new information-theoretic approach, termed the Information Guillotine (IG), towards invariant representation learning. IG closely resembles the renowned Information Bottleneck (IB) approach, but avoids ...
% We apply IB to discover environment-invariant representations. Particularly, we observed that IG demonstrated appealing performance when different environments are observed sequentially, a more practical scenario that invariant risk minimization (IRM), a recent proposal by Arjovsky et al. (2019), fails.}
 %end abstrsact

\section{Introduction}
Modern machine learning is based on the concept of independent and identical distributed (\emph{i.i.d.}) samples. However, when the test distribution is different form the training distribution, a deep neural network (DNN) trained by miminizing the  empirical risk minimization (ERM) is prone to capture spurious correlations in the training set, which in turn may lead to catastrophic performance loss during deployment.

% For example when the network learns the detect the photographer more then what is captured.

%The Information Bottleneck (IB) principle~\cite{tishby99information} as an information-theoretic framework for learning.
% tishby99information

The IB principle has been demonstrated effective in improving generalization performance of DNNs in many practical applications. It also provides a solid mathematical tool to analyze and understand the generalization behaviour of DNNs during training~\cite{shwartz2017opening,yu2019understanding}. However, existing neural network-based IB approaches, such as variational information bottleneck (VIB)~\cite{alemi2016deep}, nonlinear information bottleneck (NIB)~\cite{kolchinsky2019nonlinear} and conditional entropy bottleneck (CEB)~\cite{fischer2020conditional}, just provide empirical justifications on their superiority on either out-of-distribution (OOD) detection or generalization in different AI applications~\cite{mahabadi2021variational,kim2021drop,fischer2020ceb}, without a further explanation of the reasons behind the success of these methods.

%\textcolor{red}{In this sense, it still remains a question on the practical performance of IB approach in out-of-distribution generalization, especially when the test data comes from multiple related distributions or environments. On the other hand, it is still unclear why existing IB approaches are able to achieve the superior performances in OOD detection.}

In this work, we establish the connection between IB theory and invariant causal representation learning, and propose
a new neural network-based IB approach that generalizes very well in a sequential environment scenario. Specifically, our main contributions are fourfold:
\begin{itemize}
    % \item A new information-theoretic approach that is able to remove progressively remove irrelevant features?
    \item We establish the connection between IB theory and invariant causal representation learning, which provides theoretical justification on the generalization performances of various IB approaches in AI applications.
    \item We propose gated information bottleneck (GIB), a new IB approach that is able to progressively drop spurious correlations and select the most task-relevant features. Using the recently proposed matrix-based R{\'e}nyi's $\alpha$-order mutual information~\cite{giraldo2014measures,yu2021measuring}, GIB can be simply optimized without variational approximation and distributional assumptions.
    \item GIB outperforms other popular neural network-based IB approaches (such as VIB and NIB) in adversarial robustness and OOD detection.
    \item GIB enables reliable generalization when different environments arrive sequentially, a more practical scenario where invariant risk minimization (IRM)~\cite{arjovsky2019invariant} fails.
\end{itemize}

\section{Information Bottleneck and OOD Generalization}
\subsection{Problem Formulation}
%We assume here that there exist a mapping from the input $x \in X$ to the feature space $z = \Phi^e(x): X \to Z$, $X \subseteq R^{n}, Z \subseteq R^p$ and a classifier map for the feature space to the output space $w^e(z): Z \to Y \subseteq R^m$. The two mapping are parametrized by $\phi$ and $w$. The invariant risk minimization is defined by solving the following problem
%\begin{align} \label{eq:irm}
%    \min_{\Phi,w }  \sum_{e \in E} R^e(w \circ \Phi ) + % \lambda \sum_{e \in E}|| \nabla_w R^e(w \circ \Phi) ||
%\end{align}
%where
%\begin{align}
%R^e(w \circ \Phi) = E_{(x,y) \sim D^e } l(w \circ \Phi (x), %y )
%\end{align}
%and $D^e$ are the data from the $e \in E$ environment. A typical cost function is the cross entropy or the $\ell_2$ loss function. We want to learn a map $\phi$ which is invariant to the environment where the process is observed.

Let $\mathcal{X}$ and $\mathcal{Y}$ be the input and the label spaces. Given training set $D_e$, we obtain $N_e$ training samples $\{\mathbf{x}_i,y_i\}_{i=1}^{N_e}$ following a distribution $P_e(\mathbf{x},y)$, defined over $\mathcal{X}\times\mathcal{Y}$.
Our goal is to learn a neural network $f_\theta = w \circ \Phi$, composed of a feature extractor $\Phi$ and a classifier $w$, that generalizes well to a different, but related, unseen test set $D_t$, which follows a new distribution $P_t(\mathbf{x},y)$, i.e., minimizing the objective:
\begin{equation}
    \mathbb{E}_{(\mathbf{x},y)\sim D_t}\left[\ell(\theta;\mathbf{x},y)\right],
\end{equation}
where $\ell(\theta;\mathbf{x},y):\mathcal{W}\times \mathcal{X} \times \mathcal{Y}\rightarrow\mathbb{R}$ is the loss function of $\theta$ associated with sample $(\mathbf{x},y)$, and $\mathcal{W}\subseteq\mathbb{R}^d$ is the model parameter space. This problem is also known as out-of-distribution (OOD) generalization.

In a multi-environment setting, we have a set of training environments $E =\{e_1,e_2,\cdots,e_m\}$. Each environment $e_i$ has its own associated training data set $D_{e_i}$ drawn from distribution $P_{e_i}(\mathbf{x},y)$. Usually, all environments are assumed to be available to a learning system at the same time, which is unrealistic in numerous applications. A learning agent experiences environments often sequentially rather than concurrently.

\subsection{Information Bottleneck (IB) Approach}
IB approach considers extracting information about a target variable $Y$ (e.g., class labels) through a correlated observable $X$ (e.g., raw features). The extracted information is quantified by a variable $Z$, which is (a possibly randomized) function of $X$, thus forming the Markov chain $Y \leftrightarrow Z \leftrightarrow X$. Suppose we know the joint distribution $p(X,Y)$, the objective is to learn a representation $Z$ that maximizes its predictive power to $Y$ subject to some constraints on the amount of information that it carries about $X$:
\begin{equation}\label{eq:IB_Lagrangian}
    \mathcal{L}_{IB}=I(Y;Z) - \lambda I(X;Z),
\end{equation}
where $I(\cdot;\cdot)$ denotes the mutual information. $\lambda$ is a Lagrange multiplier that controls the trade-off between the \textbf{sufficiency} (the performance on the task, as quantified by $I(Y;Z)$) and the \textbf{minimality} (the complexity of the representation, as measured by $I(X;Z)$). In this sense, the IB principle also provides a natural approximation of \emph{minimal sufficient statistic}~\cite{gilad2003information}.

The IB principle has analytical solution when $X$ and $Y$ are either joint Gaussian~\cite{chechik2005information} or discrete~\cite{tishby2000information}. When parameterizing IB with a DNN, $Z$ refers to one of the hidden layers~\cite{dai2018compressing}.

\subsection{Generalization from a Causal Perspective}\label{sec:connection}

\begin{figure}
	\centering
    \includegraphics[width=0.7\linewidth]{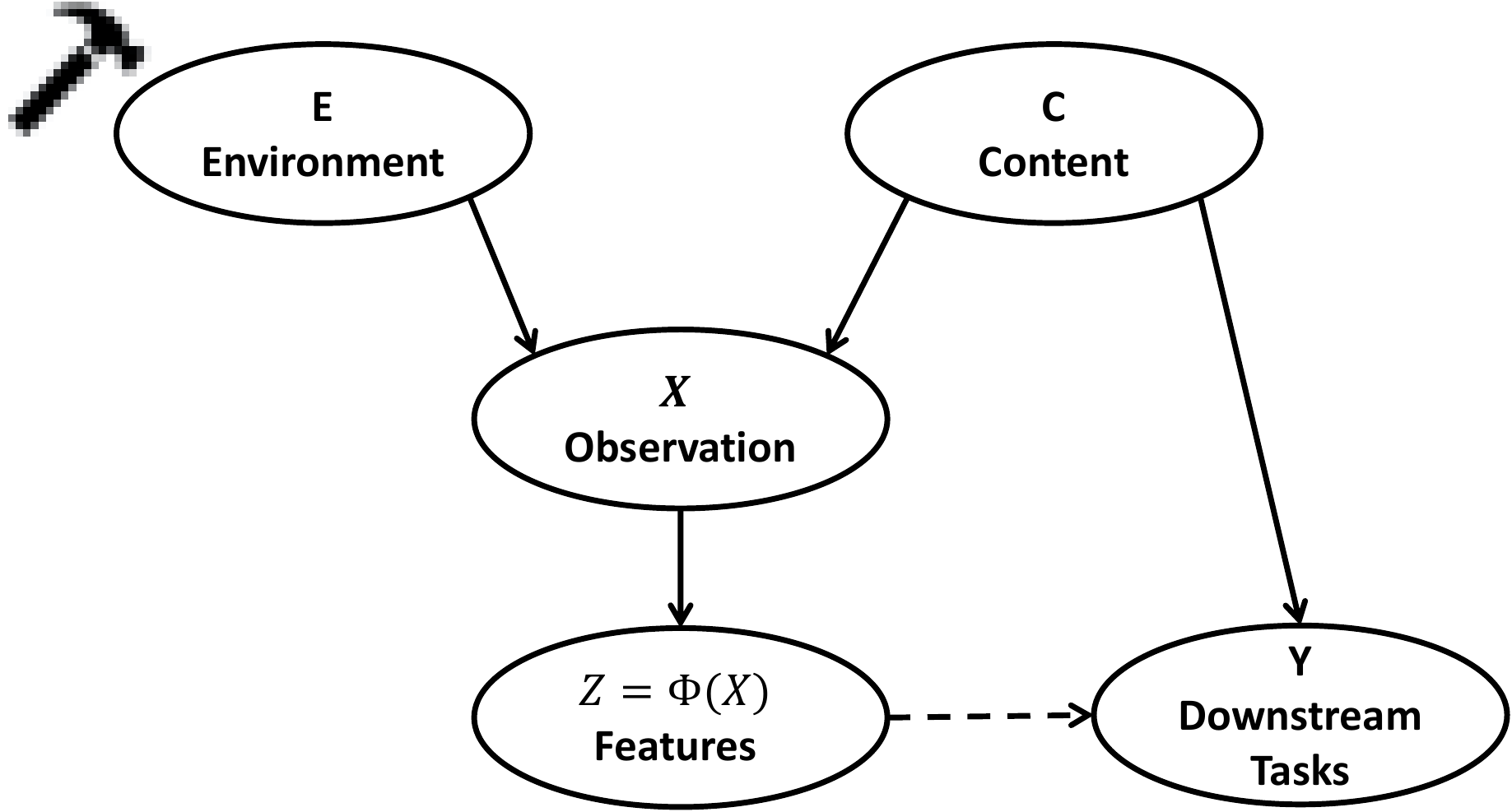}
	\caption{Causal graph: The observation $X$ is generated from the the content $C$ and the environment $E$. Content $C$ individually influences the downstream task $Y$, whereas environment $E$ does not. We aim to learn representations $Z=\Phi (X)$ from $X$ that is suitable to predict $Y$.}
	\label{fig:causal_model}
\end{figure}

When learning with multiple environments, we can model the observed data using the causal model of Fig.~\ref{fig:causal_model}~\cite{mitrovic2020representation}. From a causal prospective \citep{arjovsky2019invariant,ahuja2020invariant,pearl2009causality}, we consider each environment as an intervention in the causal graph.
We assume the observation $X$ is generated by content $C$ and the environment $E$, and only content is relevant for solving downstream tasks~\cite{mitrovic2020representation}. We  additionally assume that content and environment are independent, i.e., environment changes but content preserves. A classical example is that a cow sitting on green pastures or in an arid environment. The environment $E$ refers to the grass or the arid areas, whereas the cow or different parts of cow could be the content $C$ that ``generates" the image of the cow $X$.
We aim at learning features $Z$ (e.g. presence and shape of legs or hears of the cow) that are independent from the environment interventions (e.g. the cow shape is independent by the presence of sand on the beach). That is, $Z = \Phi(X) \independent E$. We also expect $Z$ is able to correctly predict the labels $Y$ (i.e., the cow in our example).

As a proxy to independence, practical methods learn invariant models~\citep{arjovsky2019invariant,ahuja2020invariant,peters2016causal}.
% In this work, we aim at reducing the distance of the feature across the environments.
If the learned representations are invariant under changes in $E$ or interventions, their probability distributions remain the same~\cite{moyer2018invariant}.
In this sense, suppose the features of two environments have probability distributions $p^1(z)$ and $p^2(z)$, respectively, we seek to minimize:
\begin{align}
	\min_{\phi} \KL{ p^1(z) || p^2(z)}.
\end{align}
% {\color{blue} Introduce x,x'}

While this approach requires access to full probability distribution of the features, we consider an alternative approach of minimizing the conditional distributions over the samples $x \sim p^1(x),x'\sim p^2(x)$ of the two environments, or
\begin{align}
	\min_{\phi} \E_{x,x'} \KL{ p^1(z|x) || p^2(z|x')}.
\end{align}

{\color{black}
The equivalence of these two approaches is supported by Lemma~\ref{th:kl_z_env_minimal}. All proofs are available in the Appendix.
\begin{lem} \label{th:kl_z_env_minimal}
Given two random variables $X,Z$, with two conditional distributions $p^1(z|x),p^2(z|x)$,
% \added{and same marginal $p^1(z)=p^2(z)$}
we have:
\begin{align*}
    \KL{ p^1(z) || p^2(z)} &\approx \E_{x,x'} \KL{ p^1(z|x) || p^2(z|x')}.
\end{align*}
\end{lem}

By minimizing KL divergence of the conditional distributions, we implicitly minimize the mutual information between $x$ and $z$. This property is confirmed by Lemma~\ref{th:mutual_information_env_minimal}.
\begin{restatable}[Cross-Domain Mutual Information Upper Bound]{lem}{mutualinformationenvminimal}
\label{th:mutual_information_env_minimal}
% Given three random variables $X^1,X^2,Z$,
Given two random variables $X,Z$,
with two conditional distributions $p^1(z|x),p^2(z|x)$
% \added{and same marginal $p^1(z)=p^2(z)$},
and assume the same marginal $p^1(x)=p^2(x)$, let us define the cross-domain mutual Information (CDMI) as $I^{12}(X;Z)=\KL{p^1(z,x)||p^2(z)p^2(x)}$, then the value of CDMI is upper bounded by:
\begin{align}
    I^{12}(X;Z)  & \le \E_x \E_{x'} \KL{ p^1(z|x)||p^2(z|x')}.
\end{align}
\end{restatable}
}

Combining Lemma~\ref{th:kl_z_env_minimal} and Lemma~\ref{th:mutual_information_env_minimal}, it is easy to find that the invariant representation constraint that originates from a causal perspective to generalization (i.e., $\min \KL{p^1(z) || p^2(z)}$) can be interpreted as a minimization of the upper bound of CDMI $I^{12}(X;Z)$.
On the other hand, the IB objective (especially $\min I(X;Z)$ in Eq.~(\ref{eq:IB_Lagrangian})) in the limit $\lambda \to 0$, yields a sufficient invariant representation $Z$ of the test datum $X$ for the task $Y$~\cite{achille2018emergence}.
These connections jointly provide new insights why IB principle is beneficial to OOD generalization.

\section{Gated Information Bottleneck }

{\color{black}
\subsection{Information Bottleneck in a Multi-Environment Setup}
Following Fig.~\ref{fig:causal_model} and Lemma~\ref{th:kl_z_env_minimal}, a natural objective to learn invariant representation for generalization is given by:
\begin{align} \label{eq:invariant_features}
   \min_{\Phi,w } & \sum_{e \in E} R^e(w \circ \Phi  ) \\
   &~ + \lambda  \sum_{e \in E} \sum_{e' \in E} \E_{x \sim p^e(x)} \E_{x' \sim p^{e'}(x)}  {\KL{p(z|x)||p(z|x')}}, \nonumber
\end{align}
where the first term minimizes the risk from each training environment, while the second term promotes the invariance of features across environments.

%By Lemma~\ref{th:mutual_information_env_minimal},
By the connections established in Section~\ref{sec:connection},
we can, instead, resort to minimizing the following objective:
\begin{align}\label{eq:IB_multi}
   \min_{\Phi,w } & \sum_{e \in E} R^e(w \circ \Phi  ) + \lambda  I^e(X;Z),
\end{align}
which turns out to be the information bottleneck approach in multi-environment setup.

While Eq.~(\ref{eq:IB_multi}) is general and addresses our problem, the estimation of mutual information $I(X;Z)$ is not a trivial task, especially in high-dimensional space.

In order to simplify our implementation, we implement Eq.~(\ref{eq:IB_multi}) by a deterministic encoder $X\mapsto Z$. An appealing property arises immediately~\cite{amjad2019learning}:
\begin{equation}
    I(X;Z) = H(Z) - H(Z|X) = H(Z).
\end{equation}

This is just because in a deterministic neural network, the distribution of $z$ solely depends on that of $x$, such that the uncertainty of the conditional distribution $p(z|x)$ is zero. We can therefore rewrite Eq.~(\ref{eq:IB_multi}) as:
\begin{align}\label{eq:IB_multi_entropy}
   \min_{\Phi,w } & \sum_{e \in E} R^e(w \circ \Phi ) + \lambda H^e(Z)
\end{align}

%we propose to extend the IB to multi-environment invariant feature learning. Similar to IB, the first term aims at predicting the labels from the features, while the second term minimizes the mutual information between the input and the features across the environments. This second term, not only reduces their mutual information, but also promotes the invariance of the feature maps across environments, similar to the principle of invariance introduced in \citep{arjovsky2019invariant}, where a feature map is required to induce invariant classifier for each environment. The proposed approach thus comprises of two terms: 1) the prediction loss and 2) the invariance loss

% we propose to learn the feature map $\Phi$ that, when trained in the single environment, gives rise to a invariant function.
\subsection{GIB: General Idea and a Trainable Soft Mask}

\begin{figure}
	\centering
	\includegraphics[width=0.49\textwidth]{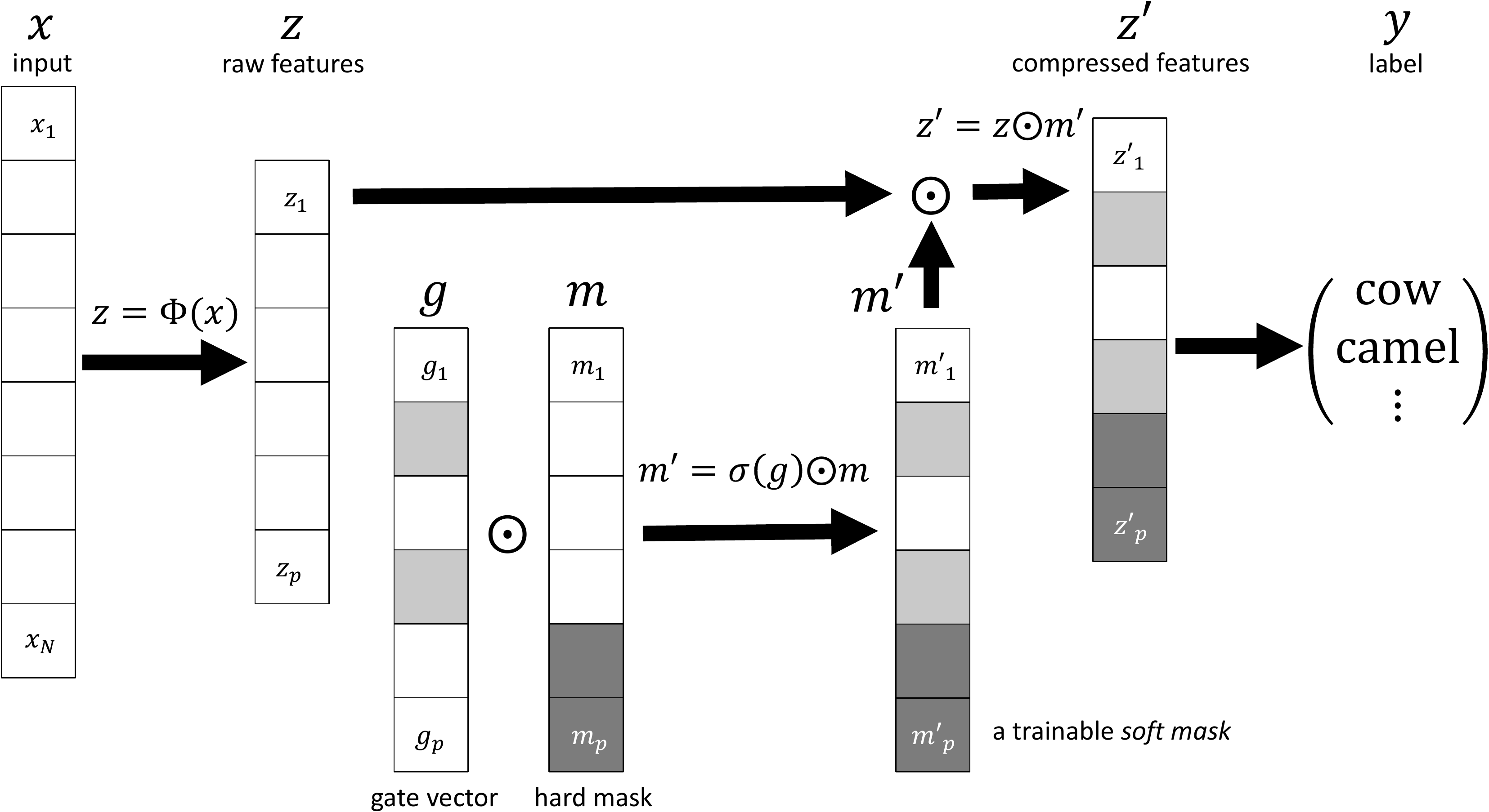}
	\caption{GIB architecture: GIB extracts compressed features $Z'$ (in a lower dimensional space) from raw features $Z$ extracted by the $\Phi$ using a trainable \emph{soft mask} $m'$, which is then fed into a classify $w$ to predict labels. The \emph{soft mask} is implemented by a gating vector $g$ and a hard mask $m$ determined by $g$.}
	\label{fig:gating}
\end{figure}

To implement Eq.~(\ref{eq:IB_multi_entropy}) and promote invariant feature map across environments, we introduce a soft gating mechanism on raw extracted features $z$ to progressively remove spurious correlations contained in $z$. We name our framework the Gated Information Bottleneck (GIB). The general architecture of GIB is
illustrated in Fig.~\ref{fig:gating}, in which:
% \footnote{Our gating mechanism computes $y = w \circ m \odot \sigma(g) \circ \Phi (x) = w(m \odot \sigma(g) \odot z')$, with $z' = \Phi (x)$.}.
% To implement the previous problem we introduce the following model
\begin{align} \label{eq:gate}
    m' &= m \odot \sigma(g) & ~~  \Phi'  &= m' \odot \Phi & ~~ z &=  \Phi'(x)
\end{align}
% \begin{align} \label{eq:invariant_features_gate}
%   \min_{\Phi,w, g }  \sum_{e \in E} R^e(w \circ (m \odot \sigma(g)) \circ \Phi  ) + \lambda H^e((m \odot \sigma(g)) \circ \Phi )
% \end{align}
% where
% \begin{align}
%     H^e((m \odot \sigma(g) ) \circ \Phi )  = E_{(x,y) \sim D^e } H((m \odot \sigma(g) ) \circ \Phi (x))
% \end{align}
The \emph{soft mask} $m'$ consists of two components: 1) a gating function $\sigma(g)$, with $g \in \mathbb{R}^p$ being a trainable vector and $\sigma: \mathbb{R} \rightarrow [0,1]$ a non-linear activation function in $[0,1]$ such as Sigmoid; and 2) a hard mask $m \in \{0,1\}^p$.
% that is determined by $\sigma(g)$ using simple heuristics.
% \textcolor{red}{[Motivation of Masking to implement Eq. (9)]
The masking component promotes the intersection of the domain of features of the environments. The first term is trained alongside the other network parameters, while the second term is used to progressively reduce the number of active neurons (that consist of spurious correlations) and is updated based on $\sigma(g)$ and its previous values.
% }
% with the form:
% \begin{equation} \label{eq:guillotine}
%     m = 1_{ \{ \sigma(g) \ge \tau \} },
% \end{equation}
% in which $\tau$ is set to be $0.5$ throughout this work.

% The intuition of the minimization problem is that, given a mask $m$, we minimize for both the loss function and the entropy of the features after the gating. The role of the gating is to suppress the features, while the mask is to reduce the information flow, such that the remaining features has to contain only relevant feature and thus removing spurious features.

% \subsection{Hard mask update heuristic}
The role of the hard mask is to avoid that features return active after been dropped out or to force feature to be zero. We implement these two strategies using
\begin{align} \label{eq:guillotine}
    m = 1_{ \{ \sigma(g) \ge \tau \} } & ~ \text{or} ~
    m = 1_{ \{ \sigma(g) \ge \tau \land m = 1 \} },
    % m = 1_{ \{ \sigma(g) \ge \tau \lor m =1 \} }
\end{align}
with $\tau$ a hyper-parameter which is set to be $0.5$ throughout this work, and $\land$ the logic {\it AND} that is used to avoid removed features to be reactivated. We also limit the minimum number of non-zero features (i.e. hard mask non-zero entries) to $\underbar{m}$ to avoid network collapse. We update $m$ every $L$ steps.

In this sense, $\sigma(g)$ can be interpreted as the probability that if a feature is selected. The combination of $\sigma(g)$ and $m$ plays the role of a \emph{soft mask} on features. For example, for the $i$-th feature (i.e., neuron), if $g_i=3$, we have $m_i=1$ and $m'=\sigma(g_i)\approx 0.95$, which indicates that the $i$-th feature is selected for the downstream task with a probability of $0.95$ and is also re-scaled by a factor of $0.95$. On the other hand, if $g_i=-3$, we have $m_i=m'_i=0$, which indicates the $i$-th feature is dropped out.
The details of GIB is presented in Algorithm~\ref{alg:IG}, in which we first train feature extractor $\Phi$, classifier $w$ and hard mask $m$ in the first environment,
and then only update $m$ (with fixed $\Phi$ and w) in subsequent environments.
%The general architecture of GIB is illustrated in Fig.~\ref{fig:gating}, while the proposed method is presented in Algorithm~\ref{alg:IG}.

%\subsection{Information Projection for Sequential Environments}

The \emph{soft mask} is also beneficial when observing environments sequentially; indeed, in accordance with our causal model Fig.~\ref{fig:causal_model}, we only need to identify task-relevant features and remove the spurious correlations, i.e. the features that are associated with
% \replaced{a specific} {the}
the current environment. After training on the first environment, we thus fix the feature extractor $\Phi$ and classifier $w$ networks' parameters and only optimize Eq.~(\ref{eq:IB_multi_entropy}), with respect to the gating trainable vector $g$, where the gating function is defined in Eq.~(\ref{eq:gate}).
% or
% \begin{align} \label{eq:info_prj_phase_ii}
%   \min_{g }  \sum_{e \in E} R^e(w \circ (m \odot \sigma(g)) \circ \Phi ) + \lambda H^e((m \odot \sigma(g)) \circ \Phi)
% \end{align}

% \begin{comment}
% \begin{align} \label{eq:invariant_features_gate}
%   & \min_{\Phi,w, g } \sum_{e \in E} R^e(w \circ (m \odot \sigma(g)) \circ \Phi  ) \nonumber  \\
%   & + \lambda  \sum_{e \in E} \sum_{e' \in E} \E_{x \sim p^e(x)} \E_{x' \sim p^{e'}(x)} \nonumber \\
%   &~ \KL{m \odot \sigma(g) \circ \Phi (x) ||m \odot \sigma(g) \circ \Phi (x') }
% \end{align}

% \subsection{Invariance Loss}
% While Eq.\ref{eq:invariant_features} is general and addresses our objectives, its computational complexity is quadratic, due to the two expectations. In order to simplify our implementation, we first notice that
% \begin{align}
%     I(x;z) &= H(z) - H(z|x)  \le \E_x \E_{x'} \KL{ p(z|x)||p(z|x')},
% \end{align}
% but, since for deterministic function the entropy $H(z|x)$ is constant, we can minimize over $H(z)$. We can now rewrite Eq.\ref{eq:invariant_features_gate} as
% \begin{align} \label{eq:info_prj_phase_i}
%   \min_{\Phi,w, g }  \sum_{e \in E} R^e(w \circ (m \odot \sigma(g)) \circ \Phi  ) + \lambda H^e((m \odot \sigma(g)) \circ \Phi )
% \end{align}
% where
% \begin{align}
%     H^e((m \odot \sigma(g) ) \circ \Phi )  = E_{(x,y) \sim D^e } H((m \odot \sigma(g) ) \circ \Phi (x))
% \end{align}
% \end{comment}

\subsection{Entropy Estimation}
We use the matrix-based R{\'e}nyi's $\alpha$-order entropy functional to estimate the entropy $H(z)$,
which is mathematically well defined and computationally efficient for large networks. Specifically, given $n$ samples drawn from $p(x)$, i.e., $\{\mathbf{x}_{i}\}_{i=1}^{n}\in \mathcal{X}$, each $\mathbf{x}_i$ can be a real-valued scalar or $d$-dimensional vector, it estimates the entropy on the eigenspectrum of a Gram matrix $K\in \mathbb{R}^{n\times n}$ ($K_{ij}=\kappa(\mathbf{x}_{i}, \mathbf{x}_{j})$ and $\kappa$ is a Gaussian kernel) as follows~\cite{giraldo2014measures}:
\begin{equation}\label{Renyi_entropy}
H_{\alpha}(A)=\frac{1}{1-\alpha}\log_2 \left(\tr {A^{\alpha}}\right)=\frac{1}{1-\alpha}\log_{2}\left(\sum_{i=1}^{n}\lambda _{i}(A)^{\alpha}\right),
\end{equation}
where $\alpha\in (0,1)\cup(1,\infty)$. $A$ is the normalized version of $K$, i.e., $A=K/\tr{K}$. $\lambda _{i}(A)$ denotes the $i$-th eigenvalue of $A$.

\subsection{Discrete gating function}
When in Eq.~(\ref{eq:gate}) we use a non-linear, but differentiable function $\sigma$, it's gradient can be automatically computed. The output of the function in this case is a continuous value in $[0,1]$.
If we want the value to assume discrete values, we can use for example the indicator function $\sigma(g) = 1_{g \ge 0}$. This function allows us to only consider the discrete values $\{0,1\}$, forcing the network to decide if the neuron is active or inactive.  This condition may be more preferable if we want to have a hard decision on which features to keep and which to remove, for example when observing a new environment. The gradient for this class of functions is zero almost everywhere. We need then to resort to gradient estimators.
% In addition to differential function $\sigma$, we allow the use of discrete function $\sigma(g) = 1_{g \ge 0}$.
In this case, we compute the gradient using the Straight Through estimator \citep{bengio2013estimating}.

%\subsection{Progressive Information Projection Algorithm}
\begin{algorithm}
% 	\SetAlgoLined
	\begin{enumerate}
	    \item {\bf Input}: Training samples $D^e$ \;
	    \item {\bf Hyper-parameters}: $\lambda$: the lagrangian parameter, $\tau$: the hard gate threshold, $L$: the update frequency, $\underbar{m}$: the minimum number of non-zero entry, $\sigma$: the non-linear function, $H$: the entropy function, $R^e$: the risk objective for the $e$-th environment,  \;
	\item {\bf GIB training on the first Environment}: \;
	\begin{enumerate}
	   % \item $\foreach (x,y) \sim D^e $
	    \item $m \gets \mathbbm{1}$, $\{g,\Phi,w\} \gets $Initialize()
	    \item Train feature extractor $\Phi$, gate $g$ and classifier $w$ networks' parameters with the objective defined in Eq.~(\ref{eq:IB_multi_entropy}),
	    the gating function is defined in Eq.~(\ref{eq:gate});
	   % }
	    \;
	   % Eq.\ref{eq:invariant_features_gate}
	    \item Progressively update the hard mask $m$ with Eq.~(\ref{eq:guillotine}). \;
	\end{enumerate}
	\item {\bf Information Projection for the subsequent Environments}: \;
	\begin{enumerate}
	    \item Fix feature extractor $\Phi$ and classifier $w$
	   % networks
	    parameters\;
	    \item
	    Minimize objective defined in Eq.~(\ref{eq:IB_multi_entropy}) w.r.t. $g$, where the gating function is defined in Eq.~(\ref{eq:gate}); \;
	   % \textcolor{red}{using Eq.(\ref{eq:IB_multi_entropy}), with Eqs.(\ref{eq:gate}),} \;
	   % Eq.\ref{eq:info_prj_phase_ii} \;
	    \item Progressively update the hard mask $m$ with Eq.~(\ref{eq:guillotine})
	\end{enumerate}	
	\item {\bf Output}: Network $w,\Phi$ and gate parameters $m,g$ \;
	\end{enumerate}
	\vspace{4mm}
	\caption{Gated Information Bottleneck (GIB)
% 	: Progressive Information Projection
	}
	\label{alg:IG}
\end{algorithm}

\section{Generalization and Adversarial Robustness}
We evaluate our Gated Information Bottleneck (GIB) in terms of generalization, robustness to adversarial attack, and out-of-distribution data detection on benchmark datasets: MNIST, FashionMNIST, and CIFAR-10. The purpose is to show the superiority of GIB over existing neural network-based IB variants, which also offers an empirical justification why GIB is able to generalize well to unseen environments, as will be tested in the next section.

\subsection{Robustness to Adversarial Attacks}
Different types of adversarial attacks have been recently proposed to ``fool" models by adding small carefully designed perturbations. In this section, we use two commonly used attacks to exam the adversarial robustness: the Fast Gradient Sign Attack (FGSM)~\cite{goodfellow2014explaining}, and the Projected Gradient Descent (PGD)~\cite{madry2017towards} which obtains adversarial examples by a multi-step FGSM (we set $\gamma =0.1$ and $t=5$ for PGD).

We provide adversarial robustness results on MNIST. We randomly select $10k$ images from the training set as the validation set for hyper-parameter tuning. For a fair comparison, we use the same architecture that has been adopted in~\cite{alemi2016deep}, namely a MLP with fully connected layers of the form $784-1024-1024-256-10$, and ReLU activation. The bottleneck layer is the one before the softmax layer, i.e., the hidden layer with $256$ units.
The Adam optimizer is used with an initial learning rate of 1$e$-4 and exponential decay by a factor of $0.97$ every $2$ epochs. All models are trained with $200$ epochs with mini-batch of size $100$.
The classification accuracy in test set is shown in Fig.~\ref{fig:MNIST_attack}.
As can be seen, our GIB outperforms other IB approaches for both types of attack. This result indicates that our method can retain more task relevant information while performing features compression.

\begin{figure}[htbp]
	\setlength{\abovecaptionskip}{0pt}
	\setlength{\belowcaptionskip}{0pt}
	\centering
	\subfigure[FGSM]{
		\includegraphics[width=0.23\textwidth]{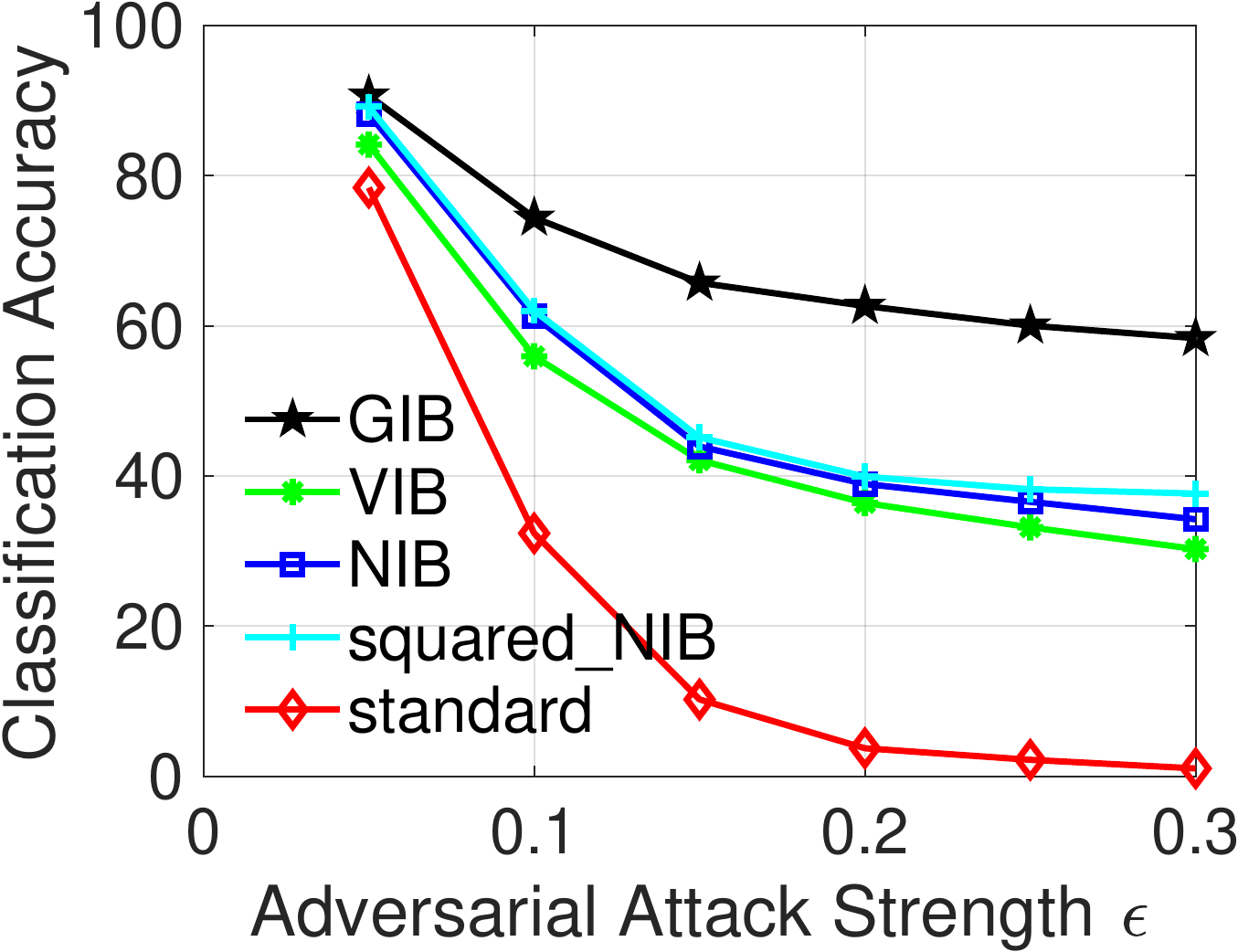}}
	\subfigure[PGD]{
		\includegraphics[width=0.23\textwidth]{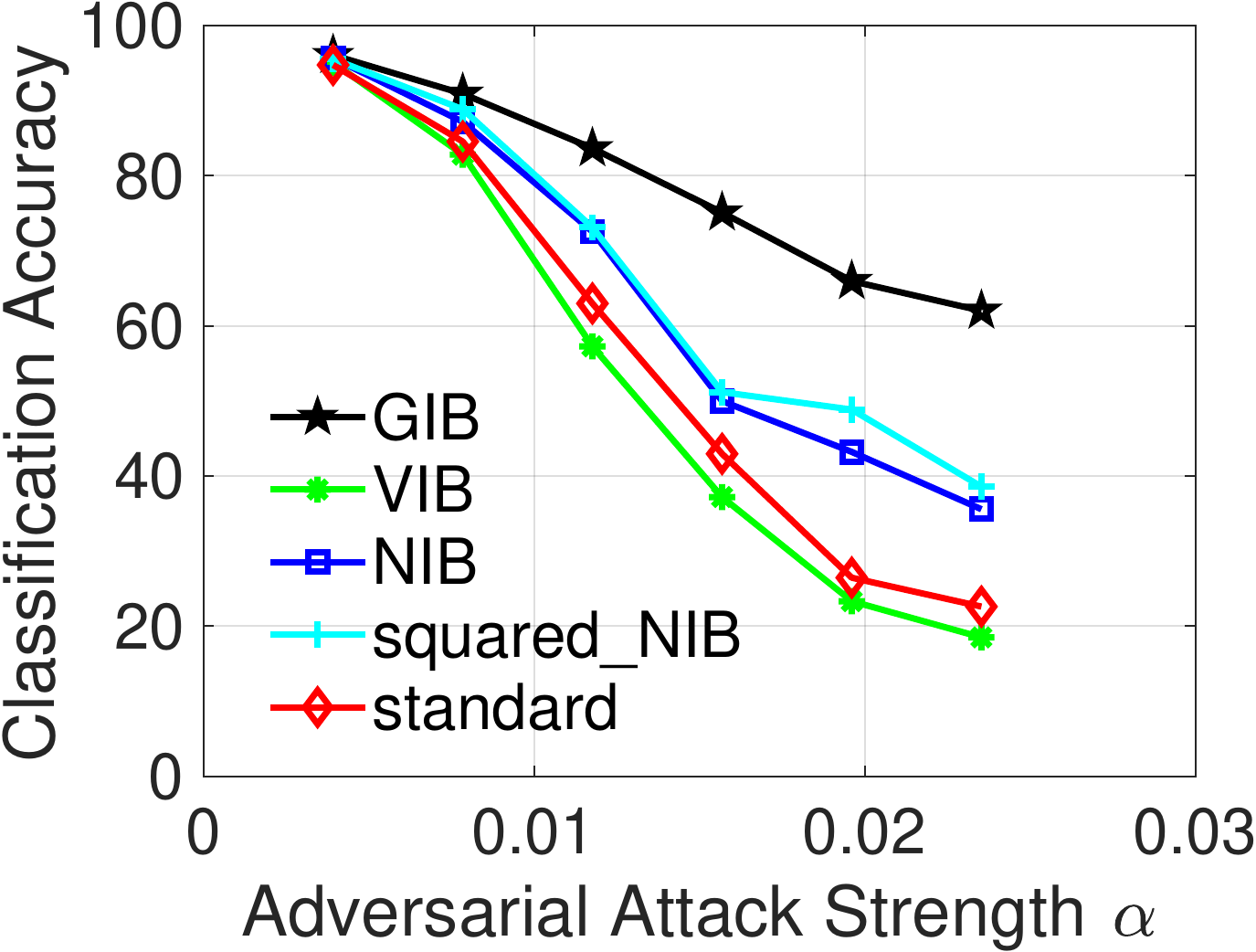}}
	\caption {Test accuracy of different IB approaches with (a) FGSM and (b) PGD on MNIST.}
	\label{fig:MNIST_attack}
\end{figure}

\begin{figure}[htbp]
	\setlength{\abovecaptionskip}{0pt}
	\setlength{\belowcaptionskip}{-5pt}
	\centering
	\includegraphics[width=0.4\textwidth]{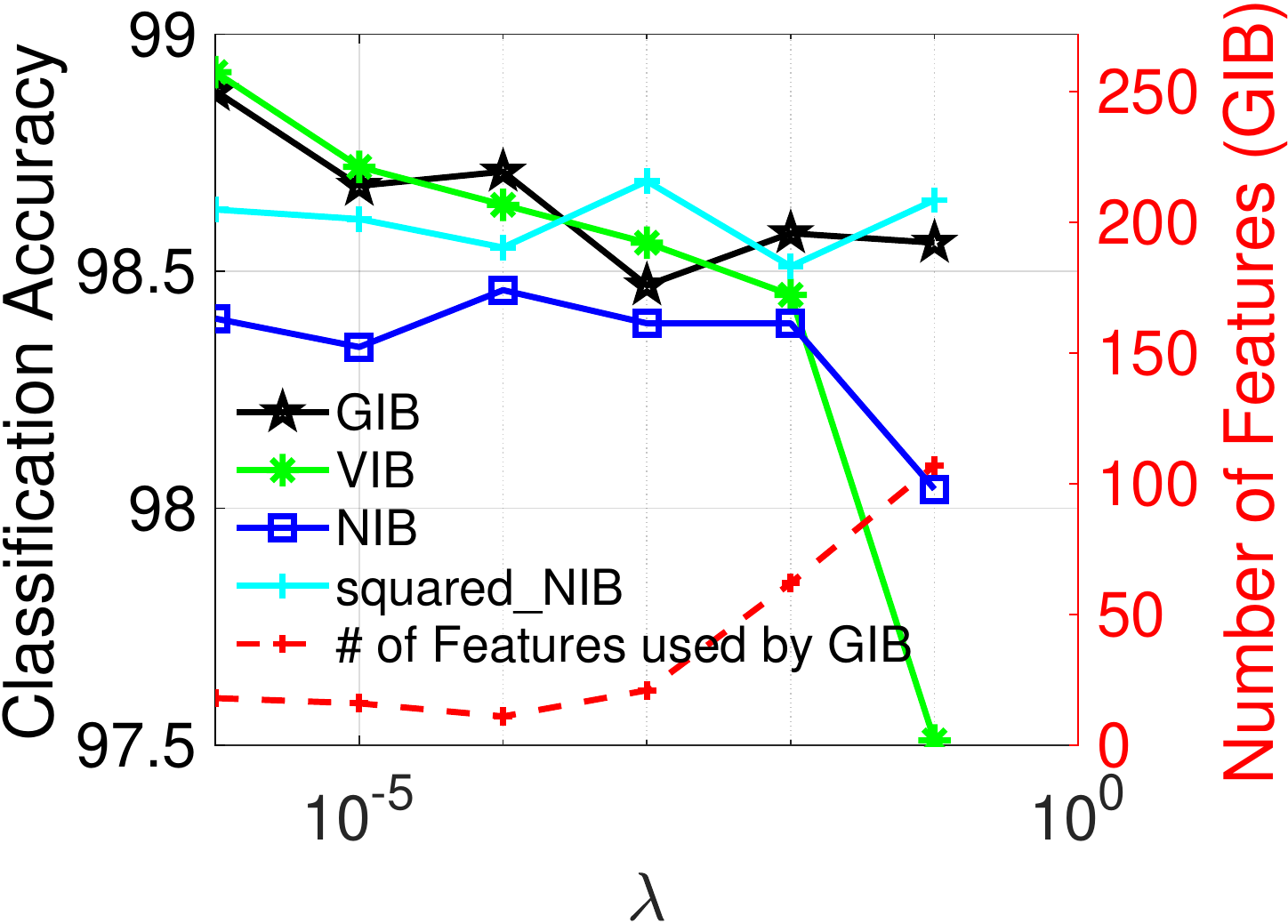}
	\caption {Test accuracy of GIB and VIB on CIFAR-10 with respect to different values of $\beta$. The dashed black curve plots the number of features used by GIB (after the masking). GIB consistently achieves the compelling classification accuracy and effectively select only a small subset of features, while other IB approaches retain $256$ features.}
	\label{fig:CIFAR_attack}
\end{figure}

\subsection{Out-of-Distribution (OOD) Detection}

%Deep neural networks (DNNs) typically  achieve high accuracy when training and testing data are from the same distribution. However, models perform worse when the test data distribution differs from the distribution of training set.

DNNs are prone to rely heavily on spurious correlations and annotation artifacts present in the training data,
while OOD examples are unlikely to contain the same spurious patterns as in-distribution examples. This phenomenon makes high-capacity DNNs are more likely to classify out-of-distribution inputs from unknown classes into known classes with high confidence. In this section, we exam our model in terms of the OOD detection and compare with other existing IB approaches.

For the problem of detecting OOD samples, we train a MLP for classifying FashionMNIST dataset. The dataset (FashionMNIST) used for training is the in-distribution (positive) dataset and the others (MNIST) are considered as OOD (negative). For evaluation, we first train a threshold based binary classifier with the confidence score of the training set, and then classify the test samples as in-distribution if the confidence score is above the threshold. Following the baseline method~\cite{hendrycks2016baseline}, we define a confidence score as the maximum value of the posterior distribution, i.e., the output of the softmax layer. We measure the following metrics: the area under the receiver operating characteristic curve (AUROC), the area under the precision-recall curve (AUPR), the false positive rate (FPR) at $95\%$ true positive rate (TPR) and detection accuracy. The metric AUPR-In and AUPR-Out denote area under the precision-recall curve where in-distribution and out-of-distribution samples are specified as positives, respectively.

The results are summarized in Table~\ref{table_OOD_detection}. Our method outperforms existing neural network-based IB approaches, which suggests that our model is more sensitive to samples that do not follow the training data distribution.

\begin{table}[h]
 \centering
  \fontsize{8}{8}\selectfont
\begin{threeparttable}
    \begin{tabular}{ccccccc}
    \toprule
    Methods & Standard & VIB & NIB & squared NIB & GIB \cr
    \midrule
    AUROC$\uparrow$ & $90.2$ & $90.6$ & $91.6$ & $92.1$ & $\bf{93.0}$ \cr
    AUPR In$\uparrow$ & $93.3$ & $92.1$ & $93.1$ & $93.1$ & $\bf{94.5}$ \cr
    AUPR Out$\uparrow$ & $90.5$ & $90.3$ & $91.3$ & $\bf{91.5}$ & $91.3$ \cr
    Detection Acc$\uparrow$ & $83.2$ & $83.2$ & $84.1$ & $84.1$ & $\bf{86.9}$ &\cr
   FPR ($95\%$ TPR)$\downarrow$ & $49.6$ & $48.0$ & $49.3$ & $49.4$ & $\bf{47.9}$ \cr
\bottomrule
\end{tabular}
\end{threeparttable}
\caption{AUROC$\uparrow$, AUPR In$\uparrow$, AUPR Out$\uparrow$, Detection Acc$\uparrow$ and FPR ($95\%$TPR)$\downarrow$ for detecting OOD inputs with density ratio and other baselines on Fashion-MNIST vs. MNIST datasets. $\uparrow$ indicates that larger value is better, $\downarrow$ indicates that lower value is better.}
\label{table_OOD_detection}
\end{table}

% \section{Experiments}
\section{Generalization in Sequential Environments}
{
\color{black} While DNN shows impressive performance in single environment, when presenting multiple environments, the training mechanism, including the loss, and the network architecture highly influence the performance in the testing environment. For example, when presenting  environments sequentially, DNNs are affected by catastrophic forgetting. While various methods have been proposed, if we do not consider memory reply based systems or generative models, even causal mechanisms, like IRM, fail to capture invariant models. In order to evaluate the performances with sequential environments, we present two experimental setups.
}
\begin{figure}
% \vspace{-4mm}
\centering
	\includegraphics[width=0.4\linewidth, trim=0cm 0cm 0cm 0cm, clip]{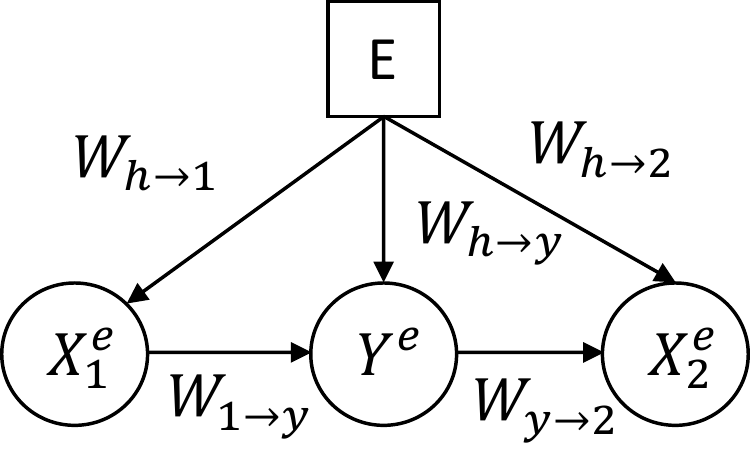}
	\caption{Synthetic experiments, defined by the matrices $W_*$ and the observed variables are $X_1^e,X_2^e$.}
	\label{fig:synthetic_causal}
    % \caption{\label{fig:experiment_causal_model} Causal relationships of colored MNIST.}
\end{figure}
\subsection{Datasets and experimental setups}
\paragraph{Synthetic Experiment}
% Following \cite{arjovsky2019invariant,creager2020environment}
% Following \cite{arjovsky2019invariant}, Tab.\ref{tab:synthetic} shows the performance in terms of Causal and non-Causal errors with the Synthetic Dataset, when the environments are observed sequentially.
Following \citep{arjovsky2019invariant}, we consider the synthetic experiment, whose aim is to identify the invariant model, but here the environments are observed sequentially. Fig.~\ref{fig:synthetic_causal} describes the causal model of the synthetic experiment, where $X_1^2,X_2^e$ are the observed variables, while $Y^e$ are the labels and $W_{h \to 1,2,y},W_{1 \to y},W_{y \to 1}$ are gaussian variables. The first $d=10$ dimensions of the model are the causal variables, while the last $d=10$ are the non-causal variables.
\begin{table}[]
    \centering
    \begin{tabular}{r|l|l}
Method	& Causal Error 	&	Non Causal Error \\	
\midrule
SEM (ground true)& 	0.00 $\pm$	0.00	& 0.00 $\pm$	0.00 \\
% \midrule
% \multicolumn{3}{c}{Parallel} \\
% \midrule
% ADMM & 18.61 $\pm$	0.70 & 	18.55 $\pm$	0.38 \\
% EIIL& 32.20 $\pm$	24.73	& 31.54 $\pm$	24.06 \\
% ERM & 83.11	 $\pm$1.54	& 82.43 $\pm$	1.67 \\
% IRM & 	66.77 $\pm$	9.30	& 63.78 $\pm$	7.85 \\
% VIRM& 61.68	 $\pm$1.71	& 16.57 $\pm$	0.39 \\
% \midrule
% \multicolumn{3}{c}{Sequential} \\
\midrule
% SADMM& 	80.78	 $\pm$7.10	& 22.04 $\pm$	11.83 \\
% SERM& 	92.29 $\pm$	2.65	& 93.11	 $\pm$2.11 \\
% SIRM& 66.63 $\pm$	2.05	& 65.08	 $\pm$1.13 \\
% PRJ& 	{ \bf 9.13 } $\pm$	0.17	& { \bf 0.10 } $\pm$	0.02 \\
SERM    & 	92.0 $\pm$ 1.9	& 92.9	 $\pm$ 1.6 \\
SIRM    &   66.2 $\pm$ 1.6	& 65.5	 $\pm$ 1.0 \\
GIB     & 	{\bf 13.7} $\pm$ 0.4	& {\bf 42.5} $\pm$ 0.7 \\
\bottomrule
    \end{tabular}
    \caption{Sequential Synthetic Dataset, Causal and non causal Errors, over $5$ repetitions; comparing with IRM and ERM computed sequentially.}
    \label{tab:synthetic}
\end{table}

% \subsection{Sequential Invariant Feature discovery}

\begin{figure}
% \vspace{-4mm}
\centering
	\includegraphics[width=0.2
% 	\includegraphics[width=0.25
% 	\textwidth]{figures/ds_16.pdf}
    \textwidth, trim=.1cm .1cm .1cm 1cm, clip]{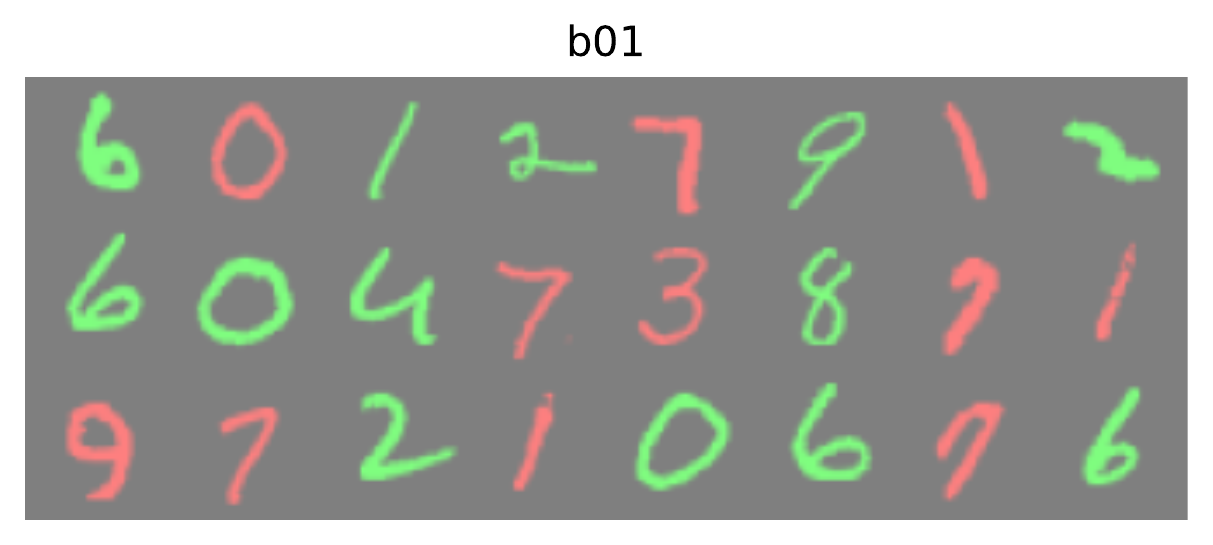}
	\caption{The Colored MNIST dataset with color correction ({\tt b01}).}
	\label{fig:color_mnist}
	\includegraphics[width=0.5\linewidth, trim=3cm 0cm 3cm 0cm, clip]{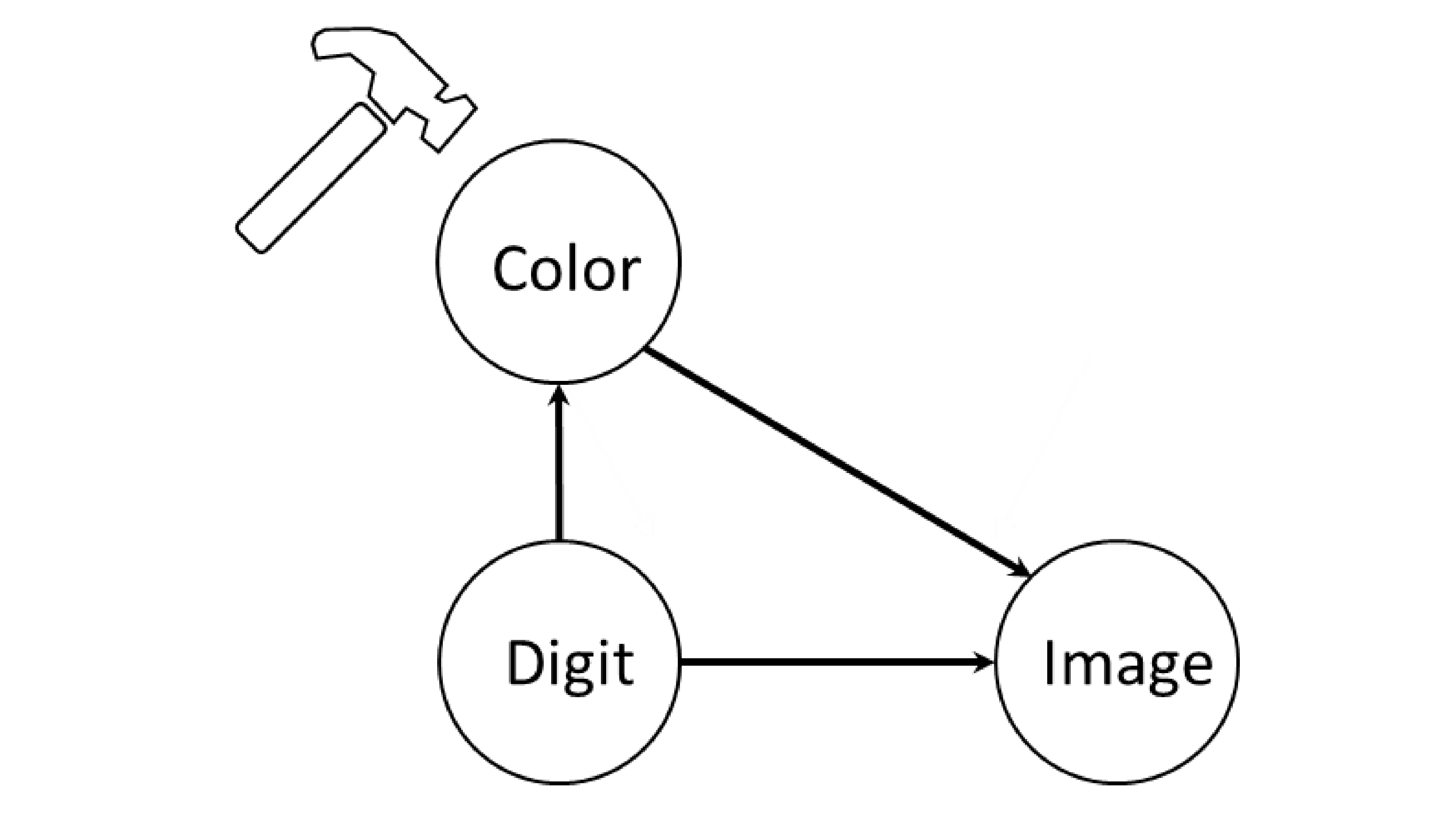}
    \caption{\label{fig:experiment_causal_model} Causal relationships of colored MNIST.}
% \end{wrapfigure}
\end{figure}

\paragraph{Colored MNIST, Fashion MNIST, KMNIST, EMNIST}
% We consider an extension of the synthetic dataset proposed in \cite{arjovsky2019invariant, ahuja2020invariant}. The task is to classify if the picture is even or odd. $|E|$ environments are created from the original dataset. Figure~\ref{fig:mnist_train_test_b01}
% Figure~\ref{fig:color_correlation} (left) shows a sample of train (upper) and test (lower) samples.
Fig.~\ref{fig:color_mnist} shows samples from the Colored MNIST dataset. In each training environment,
the task is to classify whether the digit is, respectively, even or odd. As in prior work, we add noise to the preliminary label by randomly flipping it with a probability of $p_l=0.25$. The color of the image is defined by the variable $z$, which is the noisy label flipped with probability $p_c \in [0.4,0.1]$. When multiple environments are considered, a linear probability is used in each environments, i.e., $p_c(i) = 0.4 - i 0.3/(n-1)$, where $i$ is the environment index and $n$ the total number of environments.
The color of the digit is green if $z$ is even and red if $z$ is odd. Each train environment contains $5,000$ images of size $28 \times 28$ pixels, while the test environment contains $10,000$ images where the probability $p_c = 0.9$. The color of the digit
is thus generated from the label, but depends on the environment. Fig.~\ref{fig:experiment_causal_model} depicts the causal graph (the hammer indicating the effect of the intervention of the environment). The variable ``Color" is inverted when moving from the training to test environment.

Similar to MNIST, in the Fashion-MNIST dataset, the $z$ variable defines the back-ground color and labels are assigned according to: $y = 0$ for ``t-shirt", ``pullover", ``coat", ``shirt", ``bag" and $y = 1$ for ``trouser", ``dress", ``sandal", ``sneaker", ``ankle boots". The variable $y$ is flipped with $25\%$ probability to construct the final label.
Besides, we also consider Kuzushiji-MNIST dataset~\citep{clanuwat2018deep}\footnote{\url{https://github.com/rois-codh/kmnist}} and the EMNIST Letters dataset~\citep{cohen2017emnist}\footnote{\url{https://www.nist.gov/itl/products-and-services/emnist-dataset}}. The former includes $10$ symbols of Hiragana, whereas the latter contains $26$ letters in the modern English alphabet. For EMNIST, there are $62,400$ training samples per environment and $20,300$ test samples. We set $y=0$ for letters `a', `c', `e', `g', `i', `k', `m', `o', `q', `s', `u', `v', `y' and $y=1$ for remaining ones.
The ideal model with this setup would have performance of $75\%$ on both training and test environments.

% {
% \color{blue}
% \paragraph{Cell type deconvolution: Cell composition analysis from tissue gene expressions}
% RNA sequencing (RNA-seq) is the process of measuring gene expression in tissue samples and plays a central role in biology and medicine \cite{hrdlickova2017rna}. Analysis of tissue specific gene expression is key for the advance of biological and medical research, but it is bounded to measures the average gene expression across cell types present in the sample \cite{egeblad2010tumors}. We thus consider the problem of infer cellular composition of tissues from bulk gene expression \cite{avila2018computational,wang2019bulk,menden2020deep}.
% % While single-cell RNA-seq is a reliable method, its is also costly and complex. On the other hand Bulk RNA-seq is more affordable and efficient method to characterize cellular composition of tissues.
% % method to analyze tissues' cells, whereas Bulk RNA-seq is more affordable and efficient method to characterize cellular composition of tissues.
% }

% \textbf{Reference Methods.}
\subsection{Reference Methods}
We compare with a set of popular reference methods in order to show the advantage of proposed method
% the variational Bayesian framework
in learning invariant models when observing environment sequentially.
% , and Variational Continual Learning (VCL,~\cite{swaroop2019improving, nguyen2018variational})\footnote{\url{https://github.com/nvcuong/variational-continual-learning}}.
\textbf{ERM} is the classical empirical risk minimization method; we always use the cross-entropy loss.
\textbf{IRMv1} enforces the gradient of the model with respect to a scalar to be zero.
\textbf{IRMG} models the problem as a game among environments, where each environment learns a separate model.
For completeness, we also evaluate the performances of three reference continual learning methods. These include Elastic Weight Consolidation (EWC~\citep{kirkpatrick2017overcoming}), Gradient Episodic Memory (GEM~\citep{lopez-paz_gradient_2017})\footnote{\url{https://github.com/facebookresearch/GradientEpisodicMemory}}, Meta-Experience Replay (MER~\citep{riemer2018learning})\footnote{\url{https://github.com/mattriemer/mer}}.
\textbf{EWC} imposes a regularization cost on the parameters that are relevant to the previous task, where the relevance is measured by Fisher Information (FI); \textbf{GEM} uses episodic memory and computes the updates such that accuracy on previous tasks is not reduced, using gradients stored from previous tasks;
\textbf{MER} uses an efficient replay memory and employs the meta-learning gradient update to obtain a smooth adaptation among tasks.
% \textbf{VCL*} is a modified version of the original VCL,
% \textbf{VCL} and Variational Continual Learining with coreset \textbf{VCLC} apply variational inference to continual learning.
% modified ve

\begin{table*}[]
	\centering
	\small
	\fontsize{8}{8}\selectfont
\begin{tabular}{l|ll|ll|ll}
	\toprule
	Dataset & \multicolumn{2}{c}{FaMNIST} & \multicolumn{2}{c}{KMNIST} & \multicolumn{2}{c}{EMNIST}\\
	\midrule
	Method &  train acc & test acc & train acc & test acc & train acc & test acc\\
	\midrule
    EWC &     82.7\% (0.5\%) &    24.4\%  (1.1\%) & 82.7\% (0.4\%) &    21.1\%  (1.3\%)  & 82.7\% (0.2\%) &   21.1\% (0.6\%) \\
    GEM & 82.4\% (0.3\%) &    24.5\% (1.6\%) & 83.1\% (0.4\%) &    20.8\% (1.5\%) & 82.9\% (0.6\%) &    21.3\% (0.6\%) \\
    MER & 78.7\% (0.4\%) &    19.6\% (2.1\%) & 80.9\% (0.5\%) &    20.1\% (1.7\%) &  78.8\% (1.0\%) &    19.3\%  (1.9\%) \\
    ERM & 82.7\%   (0.5\%) &    24.4\%   (1.2\%)  & 82.7\%  (0.4\%) &    21.1\%  (1.3\%) & 82.7\%   (0.2\%) &    21.1\% (0.6\%) \\
    IRM & 82.7\% (0.5\%) &    24.1\% (1.0\%) & 82.7\% (0.4\%)&    21.1\% (1.3\%) & 82.7\% (0.2\%) &    21.1\%  (0.6\%) \\
    IRMG & 84.0\%  (0.8\%) &    26.4\%  (1.4\%) & 84.3\% (0.1\%) &    23.9\%  (1.2\%)& 83.8\% (0.8\%) &    23.9\% (0.6\%)\\
    GIB & 79.9\% (5.1\%) &    {\bf 55.1\%}  (3.8\%) & 65.3\% (12.4\%) &    {\bf 47.5\%} (3.7\%) & 66.2\% (1.4\%) &{\bf 49.3\%} (2.4\%) \\
    \bottomrule
\end{tabular}
    \caption{Sequential Colored FashionMNIST, KMNIST and EMNIST Datasets with $2$ consecutive environments, over $3$ evaluations, where the probability of the first and last environments are $40\%,10\%$.}
  \label{tab:coloredMNISTseqDatasets}
\end{table*}

\begin{table*}[]
	\centering
	\small
	\fontsize{8}{8}\selectfont
\begin{tabular}{l|ll|ll|ll}
	\toprule
	Number Env. & \multicolumn{2}{c}{2} & \multicolumn{2}{c}{4} & \multicolumn{2}{c}{6} \\
	\midrule
	Method &  train acc & test acc & train acc & test acc & train acc & test acc\\
	\midrule
    EWC &     83.0\%   (0.6\%) &    24.4\%  (1.3\%) & 80.0\%  (0.4\%) &    24.7\%  (0.9\%)  & 78.9\% (0.6\%) &    22.9\% (0.3\%) \\
    GEM & 83.0\% (0.4\%) &    24.9\% (1.2\%) & 80.2\% (0.5\%) &    24.6\% (1.3\%) & 79.1\% (0.6\%) &    23.8\% (0.5\%) \\
    MER & 78.0\% (1.0\%) &    24.4\% (3.0\%) & 77.2\% (0.6\%) &    22.8\% (2.4\%) & 76.5\% (0.6\%) &    24.5\% (3.4\%) \\
    ERM & 83.0\% (0.6\%) &    24.4\% (1.3\%) & 80.1\% (0.4\%) &    24.7\% (0.8\%) & 78.9\%  (0.6\%) &    22.9\%  (0.3\%) \\
    IRM & 83.1\%  (0.6\%)  & 24.8\%  (1.3\%) & 74.8\% (0.6\%) &    16.5\% (4.3\%) & 74.7\% (0.7\%) &    13.9\% (4.6\%) \\
    IRMG & 83.8\% (0.6\%) &    26.7\% (0.6\%) & 79.4\% (0.2\%) &    28.1\% (0.9\%) & 77.2\% (0.4\%) &    27.9\% (0.5\%)\\
    GIB & 75.7\% (2.6\%) &    {\bf 55.1\%} (2.1\%) & 66.1\% (11.3\%) &    {\bf 52.9\%} (2.9\%) & 69.0\% (1.8\%) &    {\bf 53.3\%} (1.8\%) \\
    \bottomrule
\end{tabular}
    \caption{Sequential Colored MNIST Dataset in $2,4,6$ consecutive environments, over $3$ evaluations, where the probability of the first and last environments are $40\%,10\%$.}
  \label{tab:coloredMNISTseqEnvironments}
\end{table*}

\subsection{Experimental results}
\paragraph{Synthetic Dataset}
Table~\ref{tab:synthetic} presents the Causal and Non-Causal Error on the Synthetic experiments. We notice that, as expected, in the sequential environments, ERM and IRM methods are not able to capture invariant models. By enforcing consistency between the two consecutive environments, the proposed approach is able to partially recover the model components, by reducing the error from $66\%$ to $14\%$ in the causal components, while reducing from $66\%$ to $43\%$ in the non-causal components.
\paragraph{Colored Dataset}
Table~\ref{tab:coloredMNISTseqDatasets} shows the accuracy of recovering labels with two consecutive environments, evaluated in different datasets. While the hyper-parameters were selected for the MNIST datasets, the performance on the Fashion MNIST, KMNIST and EMNIST show consistent results. The training accuracy, evaluated in the two environments, is between $65\%$ and $80\%$, whereas the testing performance are between $48\%$ and $55\%$, considerably outperforming the other approaches.

Table~\ref{tab:coloredMNISTseqEnvironments} shows the accuracy in the train and test phases, when the number of environments increases from $2$ to $4$ and $6$. We notice that IRM performance reduces as the number of environments increases, while the IRMG slightly improves. The proposed approach shows a considerable advantage, especially in the test environments with values around $53\%-55\%$, while still maintaining good performance in the training environments around $75\%-66\%$.

\section{Related Works}
\subsection{Information Bottleneck: From Theory to Applications}

The IB principle has both theoretical and practical impacts to DNNs. Theoretically, it was argued that, even though the IB objective is not explicitly optimized, DNNs trained with cross-entropy loss and stochastic gradient descent (SGD) inherently solve the IB compression-prediction trade-off~\cite{tishby2015deep,shwartz2017opening}.
However, this observation is still under debate, and different mutual information estimators may lead to different behaviors regarding the existence of compression~\cite{zaidi2020information,goldfeld2020information}.

Practically, IB can be formulated as a learning objective for deep models. When parameterizing IB with a DNN, $X$ denotes input variable, $Y$ denotes the desired output (e.g., class labels), $T$ refers to the latent representation of one hidden layer. In a typical classification setup, this was done by optimizing the IB Lagrangian (i.e., Eq.~(\ref{eq:IB_Lagrangian})) via a cross-entropy loss (which amounts to $\max I(Y;T)$~\cite{achille2018information,amjad2019learning}) regularized by a differentiable mutual information term $I(X;T)$~\cite{yuinformation}. Depends on implementation details, $I(X;T)$ can be measured by the variational approximation as in the variational IB (VIB)~\cite{alemi2016deep},
the pairwise distance based estimator~\cite{kolchinsky2017estimating} as in the nonlinear IB (NIB)~\cite{kolchinsky2019nonlinear}, the mutual information neural estimator (MINE)~\cite{belghazi2018mutual} as in~\cite{elad2019direct}.

%and the matrix-based entropy functional~\cite{yu2021deep}.

%On the other hand, some previous research has proposed variants of the original IB objective. Deterministic Information Bottleneck (DIB)~\cite{strouse2017deterministic} replaces the mutual information term $I(X;Z)$ with the entropy $H(Z)$ to remove the stochasticity of the mapping $X\mapsto Z$. Squared NIB~\cite{kolchinsky2018caveats} modifies the IB objective by squaring the mutual information term $I(X;Z)$. Conditional Entropy Bottleneck (CEB)~\cite{fischer2020conditional} uses an alternative form of the original IB objective derived under the Minimum Necessary Information (MNI) criterion to preserve only a necessary amount of information.

On the other hand, variants of the original IB objective have been developed. Deterministic Information Bottleneck~\cite{strouse2017deterministic} and Conditional Entropy Bottleneck (CEB)~\cite{fischer2020conditional} replace the mutual information term $I(X;Z)$ with, respectively, the entropy $H(Z)$ or the conditional mutual information $I(X;Z|Y)$. Squared NIB~\cite{kolchinsky2018caveats} modifies the IB objective by squaring the mutual information term $I(X;Z)$. The most similar variant to us is the recently proposed Drop-Bottleneck~\cite{kim2021drop} for reinforcement learning. Although Drop-Bottleneck also aims to discretely discard redundant features that are irrelevant to tasks, it relies on a strong and over-optimistic assumption that all dimensions of the latent representation are independent. We justified the weakness of this assumption and elaborated our differences to Drop-Bottleneck in the Appendix.

%We elaborate our difference in the Appendix.

\subsection{Generalization from a Causal Perspective}

%\commentFA{Shall we cite these papers? \citep{lu2021nonlinear},\citep{kim2021drop}}

Incorporating the machinery of causality into learning models is a recent trend for improving generalization. \citep{bengio2019meta} argued that causal models can adapt to sparse distributional changes quickly and proposed a meta-learning objective
that optimizes for fast adaptation. IRM, on the other hand, presents an optimization-based formulation to find non-spurious actual causal factors to target $y$. Extensions of IRM include IRMG and the Risk Extrapolation (REx)~\citep{krueger2020out}.
Our work's motivation is similar to that of online causal learning~\citep{javed2020learning}, which models the expected value of target $y$ given each feature as a Markov decision process (MDP) and identifies the spurious feature $x_i$ if $\mathbb{E}[y|x_i]$ is not consistent to temporally distant parts of the MDP. The learning is implemented with a gating model and behaves as a feature selection mechanism and, therefore, can be seen as learning the support of the invariant model. The proposed solution, however, is only applicable to binary features and assumes that the aspect of the spurious variables is known (e.g. color). It also requires careful parameter tuning.

Independent to our work, \cite{ahuja2021invariance} also investigated the connection between IB theory and causal representation learning. Different from \cite{ahuja2021invariance} that combines the objectives of IB and IRM, our work suggests that IB alone enables the learning of invariant causal representation, and thereby improves generalization.

%\paragraph{Domain Adaptation} In \citep{zhao2020domain}, Adversarial learning is used for domain generalization. In \citep{moyer2018invariant} IB is used to build Invariant Representations, when the environment $e$ is observed.

%\paragraph{Invariant Representation}

%\paragraph{Continual Learning} \cite{kirkpatrick2017overcoming,de2019continual} learns one classifier that performs well across multiple tasks given when data is observed sequentially. The focus is on the avoidance of catastrophic forgetting. In this work, we consider continual learning when observing a single task across different environments.

\section{Conclusion}
% ``I always thought something was fundamentally wrong with the universe'' \citep{adams1995hitchhiker}

We developed gated information bottleneck (GIB), a new IB approach that directly minimizes the entropy of features selected by a trainable \emph{soft mask}. GIB is easy to optimize and encourages the learning of invariant representations over different environments. We demonstrated that GIB is more robust to adversarial attack and is more effective to distinguish in-distribution and out-distribution data than popular neural network-based IB approaches, while it also significantly reduces the feature dimension for inference. Further, we also showed that GIB is able to discover invariant causal correlations in a sequential environments scenario, which outperforms benchmark causal representation learning and continual learning methods with a large margin.

% \bibliographystyle{plain}
% I commented line 185-187
\bibliographystyle{IEEEtran}
%\bibliographystyle{named}
%\bibliography{references.bib}

{\fontsize{10}{11}\selectfont
\bibliography{main_ICDM}
}

% \bibliography{references.bib}

% \clearpage

\appendix
%\section{Supplementary Material}

%\subsection{Entropy Equalities and Inequalities}
%We present in the following sections the results on Entropic quantities, in particular on Mutual Information and KL divergence.
%\subsubsection{Mutual Information}

\subsection{Proofs to Lemma~\ref{th:kl_z_env_minimal} and Lemma~\ref{th:mutual_information_env_minimal}}

Before providing proofs to Lemma~\ref{th:kl_z_env_minimal} and Lemma~\ref{th:mutual_information_env_minimal}, let us first introduce the following three lemmas.

\begin{lem} \label{th:kl}
Given two random variables $X,Z$, whose joint distribution is given by $p(z,x)$ and conditional distribution $p(z|x)$, the KL divergence between $p(z|x)$ and $p(z)$ is upper bounded by:
\begin{align}
    \KL{ p(z|x)||p(z)}   & \le \E_{x'} \KL{ p(z|x)||p(z|x')}
\end{align}
\end{lem}

\begin{proof}[Proof of Lemma~\ref{th:kl}]
Indeed, we have:
\begin{align*}
& \KL{ p(z|x)||p(z)}  = \sum_{z} p(z|x) \ln \frac{p(z|x)}{p(z)} \\
    & =  \sum_{z} p(z|x) \ln p(z|x)  - \sum_{z} p(z|x) \ln {p(z)} \\
    & = H(z|x) - \sum_{z} p(z|x) \ln {\sum_{x'} p(z|x') p(x')} \\
    & \le H(z|x) - \sum_{x'} p(x') \sum_{z} p(z|x) \ln { p(z|x') } \\
    & = \sum_{z} p(z|x) \ln p(z|x)
    % \\ & ~
    - \sum_{x'} p(x') \sum_{z} p(z|x) \ln { p(z|x') } \\
    & = \sum_{x'} p(x') \left[ \sum_{z} p(z|x) \ln p(z|x) \right.
    % \\ & ~
    \left. - \sum_{z} p(z|x) \ln { p(z|x') }  \right] \\
    & = \sum_{x'} p(x') \left[ \sum_{z} p(z|x) \frac{\ln p(z|x)}{p(z|x') }  \right] \\
    & = \E_{x'} \KL{ p(z|x)||p(z|x')},
\end{align*}
where we used $p(z) = \sum_{x'} p(z|x') p(x') $ and $\ln {\sum_{x'} p(z|x') p(x')} \ge \sum_{x'} p(x') \ln { p(z|x') }$.
\end{proof}

%Similar to the previous result, we can consider the case of three variables $Z,X^1,X^2$ and we have the following result

\begin{lem} \label{th:kl_env}
Given two random variables $X,Z$, with two joint distributions given by $p^1(z,x),p^2(z,x)$ and conditional distributions $p^1(z|x),p^2(z|x)$, the KL divergence between $p^1(z|x)$ and $p^2(z)$ is upper bounded by:
\begin{align}
    \KL{ p^1(z|x)||p^2(z)}   & \le \E_{x'} \KL{ p^1(z|x)||p^2(z|x')}
\end{align}
\end{lem}

\begin{proof}[Proof of Lemma~\ref{th:kl_env}]
The derivation is similar to Lemma.\ref{th:kl}, where we used $p^2(z) = \sum_{x'} p^2(z|x') p^2(x') $.
\end{proof}

Given Lemma~\ref{th:kl} and Lemma~\ref{th:kl_env}, we can now introduce the main inequality in the following Lemma.
% \begin{restatable}[Mutual Information Upper Bound]{relem}{mutualinformation}
% \label{th:mutual_information}
% Given two random variables $X,Z$, whose joint distribution is given by $p(z,x)$ and conditional distribution $p(z|x)$, the Mutual Information $I(x,z)$ is upper bounded by
% \begin{align}
%     I(x,z)  & \le \E_x \E_{x'} \KL{ p(z|x)||p(z|x')}
% \end{align}
% \end{restatable}
% \begin{restatable}[Mutual Information Upper Bound]{lem}{mutualinformation}
% \label{th:mutual_information}
% Given two random variables $X,Z$, whose joint distribution is given by $p(z,x)$ and conditional distribution $p(z|x)$, the Mutual Information $I(x,z)$ is upper bounded by
% \begin{align}
%     I(x,z)  & \le \E_x \E_{x'} \KL{ p(z|x)||p(z|x')}
% \end{align}
% \end{restatable}
% \mutualinformation*

\begin{lem} \label{th:mutual_information}
Given two random variables $X,Z$, whose joint distribution is given by $p(z,x)$ and conditional distribution $p(z|x)$, the Mutual Information $I(x,z)$ is upper bounded by:
\begin{align}
    I(x,z)  & \le \E_x \E_{x'} \KL{ p(z|x)||p(z|x')}
\end{align}
\end{lem}

\begin{proof}[Proof of Lemma~\ref{th:mutual_information}]
Mutual information is the KL divergence between the join and product of marginal distributions
\begin{align}
    I(x,z) &= \KL{p(x,z)||p(x)p(z)}\\
    &= \E_x \KL{p(z|x)||p(z)}.
\end{align}

We thus build an upper bound to the mutual information:
\begin{align}
    % I(x,z) & = \E_x \KL{ p(z|x)||p(z)} \\
    I(x,z)  & \le \E_x \E_{x'} \KL{ p(z|x)||p(z|x')},
\end{align}
where we use Lemma~\ref{th:kl}.
% Indeed
% \begin{align*}
% \KL{ p(z|x)||p(z)}  &= \sum_{z} p(z|x) \ln \frac{p(z|x)}{p(z)} \\
%     &=  \sum_{z} p(z|x) \ln p(z|x) - \sum_{z} p(z|x) \ln {p(z)} \\
%     &= H(z|x) - \sum_{z} p(z|x) \ln {\sum_{x'} p(z|x') p(x')} \\
%     & \le H(z|x) - \sum_{x'} p(x') \sum_{z} p(z|x) \ln { p(z|x') } \\
%     & = \sum_{z} p(z|x) \ln p(z|x) \\
%     & ~ - \sum_{x'} p(x') \sum_{z} p(z|x) \ln { p(z|x') } \\
%     & = \sum_{x'} p(x') \left[ \sum_{z} p(z|x) \ln p(z|x) \right.  \\
%     & ~ \left. - \sum_{z} p(z|x) \ln { p(z|x') }  \right] \\
%     & = \sum_{x'} p(x') \left[ \sum_{z} p(z|x) \frac{\ln p(z|x)}{p(z|x') }  \right] \\
%     & = \E_{x'} \KL{ p(z|x)||p(z|x')}
% \end{align*}
% where we used $p(z) = \sum_{x'} p(z|x') p(x') $ and $\ln {\sum_{x'} p(z|x') p(x')} \ge \sum_{x'} p(x') \ln { p(z|x') }$
\end{proof}

Now we provide proofs to Lemma~\ref{th:kl_z_env_minimal} and Lemma~\ref{th:mutual_information_env_minimal}, respectively.

% ================================================================

% ================================================================

{\color{black}
\subsubsection{Proof to Lemma~\ref{th:kl_z_env_minimal}}
This inequality can be obtained by expanding the two distributions and using the convexity property of the logarithm:
{\small
 \begin{align*}
     & \KL{ p^1(z) || p^2(z)} = \E_{p^1(z)} \ln \frac{p^1(z)}{p^2(z)} \\
     &=\sum_{z} p^1(z) \ln p^1(z) - \sum_{z} p^1(z) \ln p^2(z) \\
     &=\sum_{x} p^1(x) \sum_{z} p^1(z|x) \ln \sum_{x'} p^1(x') p^1(z|x')
    %  \\ &~
     - \sum_{z} p^1(z) \ln p^2(z) \\
      & \ge  \sum_{x} p^1(x) \sum_{x'} p^1(x') \sum_{z} p^1(z|x) \ln  p^1(z|x')
    %   \\ &~
      - \sum_{z} p^1(z) \ln p^2(z) \\
      & =  \sum_{x} p^1(x) \sum_{x'} p^1(x') \sum_{z} p^1(z|x) \ln  p^1(z|x')
      \\ &~
      - \sum_{x} p^1(x) \sum_{z} p^1(z|x) \ln p^2(z) \\
      & =  \sum_{x} p^1(x) \sum_{x'} p^1(x') \sum_{z} p^1(z|x) \ln  p^1(z|x') \\ &~  - \sum_{x} p^1(x) \sum_{x'} p^1(x') \sum_{z} p^1(z|x) \ln p^2(z) \\
      & =  \sum_{x} p^1(x) \sum_{x'} p^1(x') ( \sum_{z} p^1(z|x) \ln  p^1(z|x')
    %   \\ &~
      - \sum_{z} p^1(z|x) \ln p^2(z) )   \\
      & =  \sum_{x} p^1(x) \sum_{x'} p^1(x')  ( \sum_{z} p^1(z|x) \ln  p^1(z|x') \\ &~  - \sum_{z} p^1(z|x) \ln \sum_{x''} p^2(x'') p^2(z|x'') )   \\
      & \approx  \sum_{x} p^1(x) \sum_{x'} p^1(x')  ( \sum_{z} p^1(z|x) \ln  p^1(z|x') \\ &~  - \sum_{x''} p^2(x'') \sum_{z} p^1(z|x) \ln  p^2(z|x'') )    \\
      & =  \sum_{x} p^1(x) \sum_{x'} p^1(x') \sum_{x''} p^2(x'') \left[ \sum_{z} p^1(z|x) \ln  p^1(z|x') \right.  \\ &~  \left. - \sum_{z} p^1(z|x) \ln  p^2(z|x'') \right]    \\
      & =  \sum_{x} p^1(x) \sum_{x'} p^1(x') \sum_{x''} p^2(x'')  \left[  \sum_{z} p^1(z|x) \ln  \frac{p^1(z|x')} { p^2(z|x'')}  \right]     \\
      & =  \sum_{x} p^1(x) \sum_{x'} p^1(x') \sum_{x''} p^2(x'') \KL{ p^1(z|x') || p^2(z|x'')}     \\
      & =  \E_{x\sim p^1(x),x'\sim p^1(x),x'' \sim p^2(x)}  \KL{ p^1(z|x') || p^2(z|x'')} \\
      & \approx  \sum_{x} p^1(x) \sum_{x''} p^2(x'') \KL{ p^1(z|x) || p^2(z|x'')}     \\
      & =  \E_{x,x'}  \KL{ p^1(z|x) || p^2(z|x')}
 \end{align*}
}where we used $\ln \sum_{x'} p^1(x') p^1(z|x')  \ge \sum_{x'} p^1(x') \ln  p^1(z|x') $,  $\sum_{x'} p^1(x') \sum_{z} p^1(z|x') = p^1(z)$ and $\sum_{x'} p^1(x') = 1$.  In the last approximation we substitute $p^1(z|x') \to  p^1(z|x)$.

\subsubsection{Proof to Lemma~\ref{th:mutual_information_env_minimal}}
Similar to Lemma~\ref{th:mutual_information}, this property follows from Lemma~\ref{th:kl_env}. Indeed, when we consider the cross-domain mutual information between $X,Z$ on the two distributions, we have:
\begin{align}
    I^{12}(X;Z) &= \KL{p^1(z,x)||p^2(z)p^2(x)}\\
     &= \E_x \KL{p^1(z|x)||p^2(z)}   \\
     & \le \E_x \E_{x'} \KL{ p^1(z|x)||p^2(z|x')},
\end{align}
where $\KL{p^1(z,x)||p^2(z)p^2(x)} = \E_x \KL{p^1(z|x)||p^2(z)}$, when $p^2(x)=p^1(x)$ and  $\KL{ p^1(z|x)||p^2(z)} \le \E_{x'} \KL{ p^1(z|x)||p^2(z|x')} $ from Lemma~\ref{th:kl_env}.

}

%\end{proof}

\subsection{Simulation Details}
\subsubsection{Synthetic Experiment}
In this experiment, the features have fix size of $d = 20$. Hyper-parameter search was done using grid search method over the following values $\lambda \in [0, 1e-5, 1e-4, 1e-3, 1e-2, 1e-1]$,
% using the $HETERO=2$ SEM,
$5$ repetitions, lr $=1e-4$, 10'000 iterations, 1'000 samples. Experiment performed on a server with $8$ CPUs, 64Gb RAM, one GPU with 8Gb RAM.
\subsubsection{Colored dataset Experiment}
Experiments have been conducted using a MLP with $3$ hidden linear layer with drop-out, each of $200$ neurons and relu activation function. Cross Entropy is used as minimization loss.  Hyper-parameters search based on grid search. For IRM and IRMG the best parameter suggested in the original works have been used. Regularization terms have been selected based on heuristics. For each environment has $5'000$ samples, while test has $10'0000$ samples, trained over $100$  epochs. For MER, EWC and GEM setting are those from the authors, when possible.
% used as suggested in the provided code and papers from the authors.
% For GIB we use a learning rate of $0.003$.

% {\color{blue}
% \subsubsection{Cellular Deconvolution dataset Experiment}
% We considered $4$ datasets, prepared following the experiments of \cite{menden2020deep}.
% \begin{figure}
% % \vspace{-4mm}
% \centering
% 	\includegraphics[width=1\linewidth, trim=0cm 0cm 0cm 0cm, clip]{figures/bulk_100.png}
% 	\caption{Cellular Deconvolution dataset. Each dataset has a different distribution of the cell types and some cell types are not present in all datasets. }
% 	\label{fig:bulk_100}
%     % \caption{\label{fig:experiment_causal_model} Causal relationships of colored MNIST.}
% \end{figure}
% }

\subsection{Analysis of the complexity (time, space, sample size)} The proposed GIB algorithm requires only marginal complexity in time and space. The space complexity is two times the size of features, marginal w.r.t the number of variable in a typical DNN.

\subsection{The Independence Assumption in Drop-Bottleneck}
Although Drop-Bottleneck~\cite{kim2021drop} shares similar ideas to ours, i.e., discretely drops redundant features that are irrelevant to tasks. We emphasize two significant differences here:
\begin{itemize}
    \item Drop-Bottleneck does not build the connection between IB and invariance representation learning. It targets reinforcement learning and implements the maximization of mutual information $I(z;y)$ by the mutual information neural estimator (MINE) as has been used in the famed deep infoMax. By contrast, same to other popular IB approaches, we simply implement this term by cross-entropy.
    \item Drop-Bottleneck assumes the independence between each dimension of the representation $Z$ for simplicity (i.e., $Z_1\independent Z_2 \independent ... Z_p$), and claims that the total dependence among ${Z_1,Z_2,...,Z_p}$  could decrease to zero as the optimization progresses.
\end{itemize}

Although it is hard to directly compare GIB with Drop-Bottleneck, we additional provide two empirical justifications why the independence assumption is harmful. Fig.~(\ref{fig:appendix}) shows the total dependence when we train a MLP on MNIST. As can be seen, the total dependence (measured by state-of-the-art $T_\alpha^*$~\cite{yu2021measuring} and IDD~\cite{romano2016measuring}) is far away from zero, which indicates that the pairwise independence assumption does not hold across training. Fig.~(\ref{fig:appendix}) shows the result when we replace our entropy estimator (i.e., Eq.~(12)) with a simpler Shannon entropy estimator that assumes fully independence (i.e., $H(z)=\sum_i^p H(z_i)$). As can be seen, such assumption significantly degrade the performance.

\begin{figure}[htbp]
	\setlength{\abovecaptionskip}{0pt}
	\setlength{\belowcaptionskip}{0pt}
	\centering
	\subfigure[total dependence]{
		\includegraphics[width=0.23\textwidth]{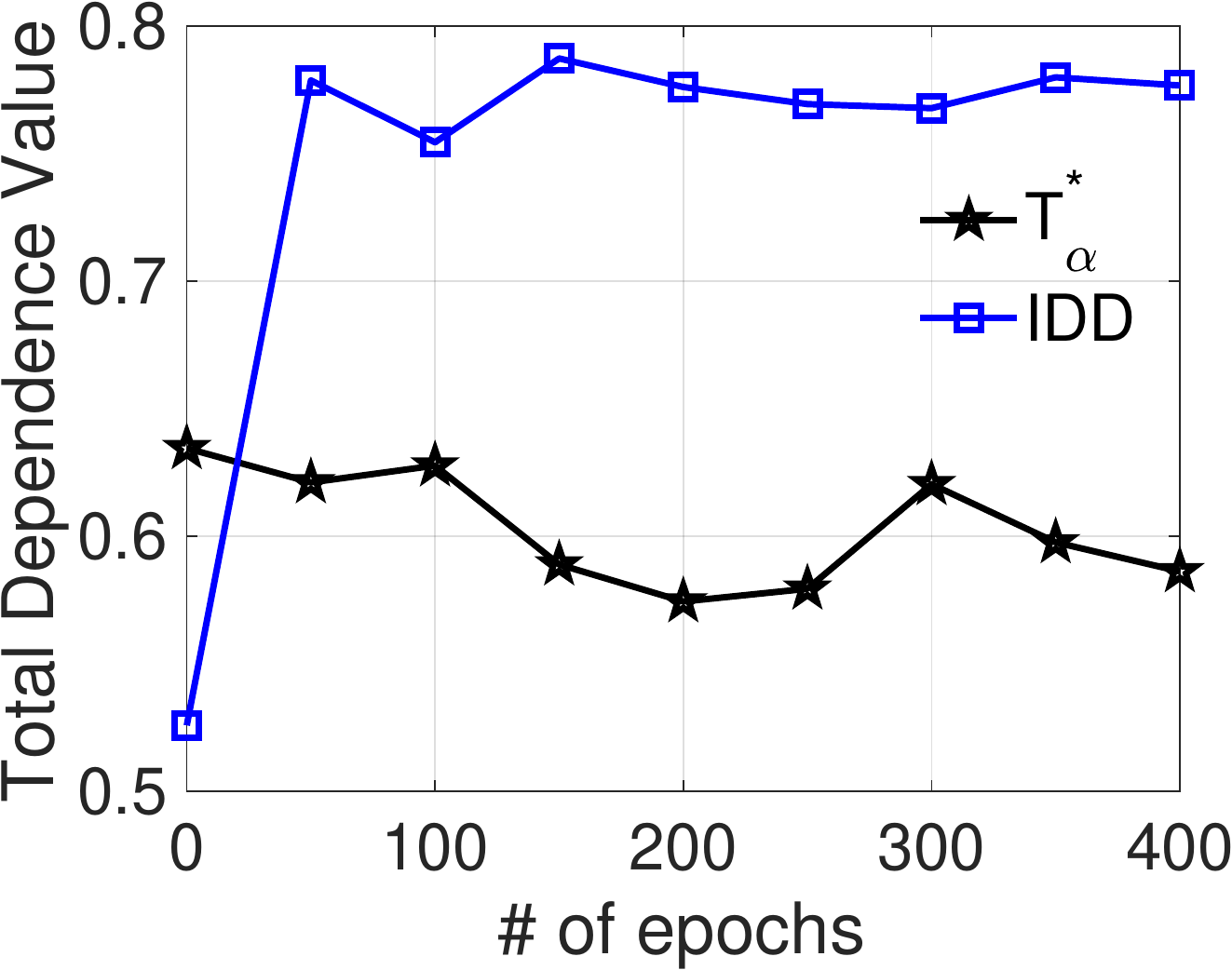}}
	\subfigure[adversarial robustness (PGD)]{
		\includegraphics[width=0.23\textwidth]{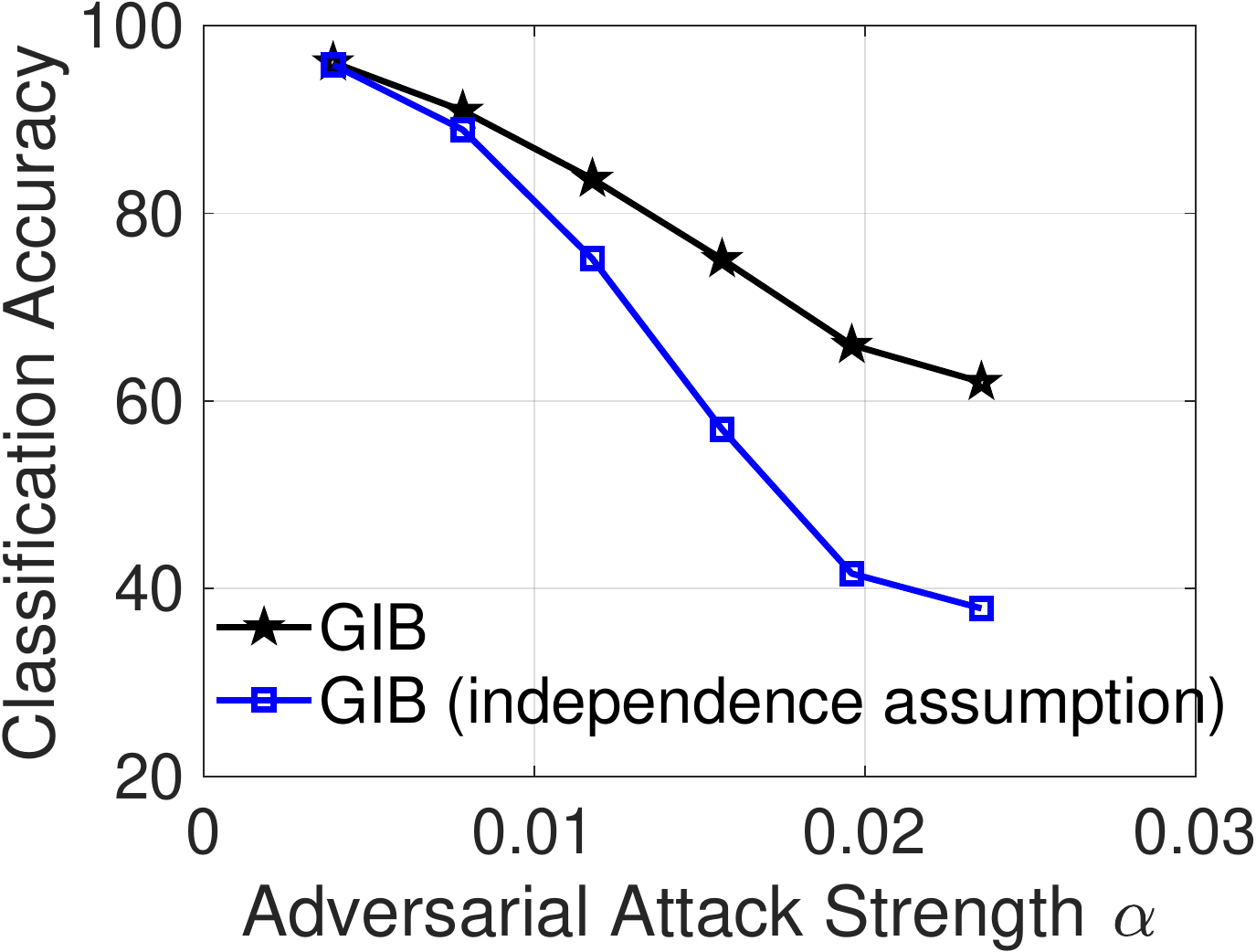}}
	\caption {(a) Total dependence value amongst each dimension of representation measured by $T_\alpha^*$ and IDD; and (b) PGD on MNIST.}
	\label{fig:appendix}
\end{figure}

\end{document}